\documentclass[10pt,journal,compsoc]{IEEEtran}
\usepackage{graphicx}
\usepackage{stfloats}
\usepackage{grffile}
\usepackage{subfigure}
\usepackage{epstopdf}
\usepackage{float}
\usepackage{amsmath,bm}
\usepackage{amsthm}
\usepackage{amsfonts}
\usepackage{booktabs}
\usepackage{algorithm}
\usepackage{algorithmic}
\usepackage{multirow}
\usepackage{url}
\usepackage[colorlinks,urlcolor=red,linkcolor=red]{hyperref}
\usepackage{amssymb}
\newtheorem{theorem}{Theorem}
\newtheorem{lemma}{Lemma}

\ifCLASSOPTIONcompsoc
  \usepackage[nocompress]{cite}
\else
  \usepackage{cite}
\fi
\ifCLASSINFOpdf
\else
\fi
\begin{document}
%
\title{A Concise yet Effective Model for Non-Aligned Incomplete Multi-view and Missing Multi-label Learning}
%
%
%
%
%
%
\author{Xiang Li,
        Songcan Chen
\IEEEcompsocitemizethanks{\IEEEcompsocthanksitem Xiang Li and Songcan Chen are with College of Computer Science and Technology/College of Artificial Intelligence, Nanjing University of Aeronautics and Astronautics, Nanjing, 211106, China and also with MIIT Key Laboratory of Pattern Analysis and Machine Intelligence. Corresponding author is Songcan Chen. \protect\\
E-mail: \{lx90, s.chen\}@nuaa.edu.cn.}
\thanks{Manuscript received August 5, 2020.}}

\markboth{Journal of \LaTeX\ Class Files,~Vol.~14, No.~8, August~2015}%
{Shell \MakeLowercase{\textit{et al.}}: Bare Demo of IEEEtran.cls for Computer Society Journals}
%



\IEEEtitleabstractindextext{%
\begin{abstract}
In reality, learning from multi-view multi-label data inevitably confronts three challenges: missing labels, incomplete views, and non-aligned views. Existing methods mainly concern the first two and commonly need multiple assumptions to attack them, making even state-of-the-arts involve at least two explicit hyper-parameters such that model selection is quite difficult. More toughly, they will fail in handling the third challenge, let alone addressing the three jointly. In this paper, we aim at meeting these under the least assumption by building a concise yet effective model with just one hyper-parameter. To ease insufficiency of available labels, we exploit not only the consensus of multiple views but also the global and local structures hidden among multiple labels. Specifically, we introduce an indicator matrix to tackle the first two challenges in a regression form while aligning the same individual labels and all labels of different views in a common label space to battle the third challenge. In aligning, we characterize the global and local structures of multiple labels to be high-rank and low-rank, respectively. Subsequently, an efficient algorithm with linear time complexity in the number of samples is established. Finally, even without view-alignment, our method substantially outperforms state-of-the-arts with view-alignment on five real datasets.
\end{abstract}

}

\maketitle

\IEEEdisplaynontitleabstractindextext

%
\IEEEpeerreviewmaketitle

\IEEEraisesectionheading{\section{Introduction}  \label{sec1}}

%
%
%
%
\IEEEPARstart{M}{ulti-view} multi-label learning is designed to predict the multiple labels of an object represented by multiple views. In reality, multi-view multi-label learning has wide applications ranging from image classification \cite {r1,r2} to video analysis \cite {r3,r4} in that multi-view multi-label data is ubiquitous. For example, an image can be described by Histogram of Oriented Gradients (HOG), Scale Invariant Feature Transform (SIFT), and color features, meanwhile, the image can also be labeled with “tree, water, sky”; a video includes diverse representations such as audio, text, and picture, at the same time, the video can be annotated by several labels like “Shakespeare, opera, King Lear”. As two individual research fields, multi-view learning \cite {r5,r2017nie,r6,r7} and multi-label learning \cite {r8,r9,r10,r11} are severally and extensively studied in last two decades. Nonetheless, to date, as their intersection or marry, multi-view multi-label learning \cite {r12,r13} is still relatively under-studied. Note that, there has been active literature motivated for “multi-modal learning” \cite {r14,r15,r16,r17,r18}, and the scope of “multi-view learning” is more extensive since the latter contains not only the multi-modal learning but also the learning paradigm of the same modal but different perspectives.

In practice, a situation we will often suffer is that multi-view multi-label data appears in three forms: missing labels, incomplete views, and non-aligned views. Reasons behind the appearance of these issues are: (1) insufficient resources or limited knowledge makes it expensive to obtain all the relevant labels of a sample; (2) malfunction of sensors or occlusion in some views causes the incompleteness of views; (3) the completely aligned information can hardly be accessed for privacy protection or aligned views are disturbed by carelessness of mankind. All of the three issues could dramatically degenerate the performance.

Existing multi-view learning methods, no matter whichever supervised \cite{r19,r20,xu2014}, semi-supervised \cite{r21,r22,nie2017}, or unsupervised \cite{r23,r24,nie2018}, commonly assume that the views involved are aligned, however, this does not necessarily always hold in reality. For example, in the social network, users may register multiple accounts as in Facebook, Twitter, and Instagram, but it is hard to align these social accounts with the same user owing to the privacy protection; in disease diagnosis, we obtain different types of examination data of the patients from different hospitals, but for the same sake of privacy protection, we cannot align these data with the same patient; in questionnaire survey, different organizations survey the same group of people, but the survey is generally anonymous, making the alignment information unavailable. In addition, non-aligned views are natural in recommendation system \cite{r25}, video surveillance \cite{r26} and so on.

It is worth emphasizing that the non-aligned views add extra challenges to the original considerably challenging problem with missing labels and incomplete views. The reasons are as follows: (1) The non-aligned views make the interactions among different views no longer easy to be available, thus their explicitly complementary information can hardly be exploited. (2) In conventional incomplete multi-view learning, those missing views of samples can be completed with the help of the paired ones from other observed views, however, in the non-aligned views setting, no paired sample can be available. As a result, the problem of multi-view incompleteness is more severe and hard to be tractable even if possible. (3) Under the situation of non-aligned views, information able to be utilized for the correspondence among views most probably hides in the common or shared labels of samples from different views, however, unfortunately, in the missing multi-label setting, such beneficial label information is quite limited.

\noindent {\textbf {Remark 1.}} If the multiple labels are complete during training, samples of the non-aligned views can be simply aligned by directly using the given labels and then divided into the corresponding groups. We refer interested readers to \cite{r27} for the partially non-aligned two views. Consequently, the non-alignment of multiple views under the situation of complete multiple labels is formally relatively trivial. However, it is worth emphasizing that in the case of missing multiple labels, the complete ground truth of the samples is no longer available. Therefore, beneficial label information for aligning the non-aligned views is quite limited, making the problem of non-aligned multiple views with the missing multiple labels more challenging and non-trivial any longer. To alleviate such insufficiency of multiple labels, extra and accurate structural information of multiple labels accompanied the non-aligned incomplete multiple views needs to be mined, which is the main focus of this work.

\begin{figure*}[htbp]
\centering
\includegraphics[width=1\textwidth]{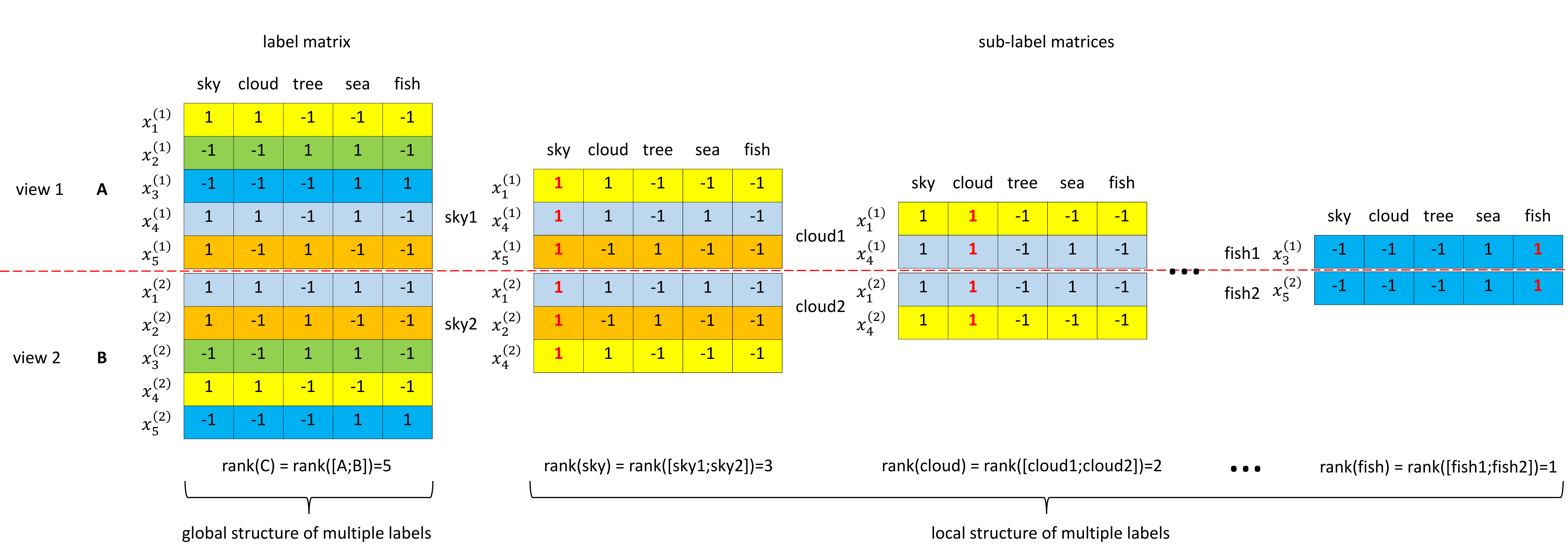} 
\caption{The global and local structures of the multiple labels. The same sample of non-aligned views is represented by the same color, and ”sky”, ”cloud” , … ,”fish” are the labels. ”1” in the label matrix means that the sample is annotated with the corresponding label whereas “-1” means not. All the label matrices are vertically concatenated. The label matrix of all samples from two views has full column rank or high rank, which refers to the global structure of multiple labels. Meanwhile, the sub-label matrix comprised of samples that share the same individual label tends to have low rank, e.g., the rank of sub-label matrix that share the same label “cloud” equals to 2, which corresponds to the local structure of multiple labels.}
\label{fig1}
\end{figure*}

To date, existing methods \cite{r28,r29} of addressing the three challenges (\textit {missing labels, incomplete views, and non-aligned views}) usually suffer from two disadvantages. First, they only focus on the first two challenges and assume that the views involved must be aligned. However, this assumption does not necessarily always hold in reality as above-mentioned. Confronting non-aligned views, traditional multi-view learning methods operating in the aligned cases are difficult to be directly adopted due to the lack of straightforward interactions among views. Second, introducing overmuch hyper-parameters in dealing with the problems of incomplete views and missing labels makes them difficult to optimize and reproduce.

To alleviate the aforecited two disadvantages, two questions arise naturally. One is how to simultaneously address all the three challenges. The other is how to efficiently build a model with as few hyper-parameters as possible.

In this paper, we propose a \textbf{N}on-\textbf{A}ligned \textbf{I}ncomplete \textbf{M}ulti-view and \textbf{M}issing \textbf{M}ulti-label \textbf{L}earning method abbreviated as NAIM$^3$L to address the arising two questions. Our starting point is that although samples among views are not aligned explicitly, they can still be bridged implicitly through the common or shared labels and thus can be learned complementally. Besides, intuitively, the samples with similar labels are more prone to be strongly correlated with each other whereas those with dissimilar labels are weakly correlated, or even uncorrelated, which reflects the local and the global structural relations within multiple labels, respectively. Mathematically, the local correlated structure can approximately correspond to the low rankness of the label matrix of samples sharing the same individual label while the global weakly correlated or uncorrelated structure can roughly correspond to the high rankness of the label matrix of all samples. An intuitive description of this global-local structure is illustrated in Fig. \ref {fig1}. For conciseness and due to the limitation of space, we only use two views to illustrate the global and the local structures of multiple labels, but such an illustration can be directly generalized to the case of more than two views. 

Note that in traditional multi-label learning methods, the label matrix is often assumed to be low-rank rather than high-rank as we do. We argue from the following two perspectives that the high-rank assumption is reasonable. First, intuitively, samples in real datasets with multiple labels are usually diverse and contain dissimilar labels. As samples with dissimilar labels are weakly correlated, or even uncorrelated, thus the entire multi-label matrix is often of a high rank. Second, mathematically, as entries of the label matrix take binary values, it is unlikely for this matrix to be low-rank. To further demonstrate the rationality of our assumptions, we show the high-rankness of the entire label matrix corresponding to all samples and the low-rankness of each sub-label matrix corresponding to samples that share a single label in Fig. \ref {Fig2a} and \ref {Fig2b}, respectively. It is well known that the rank of a matrix is equal to the number of its non-zero singular values. From Fig. \ref {Fig2a}, we can find that the singular values of the training and testing label matrices have a heavy tail, which indicates the high-rankness of the entire label matrix (the numbers of non-zero singular values of these two matrices are 259 and 249, i.e., the ranks are 259 and 249, either full-rank or almost full-rank). As the number of multiple labels is too big to show the low-rankness of each sub-label matrix in a single picture, alternatively, we show the mean value and the median value of the ranks of these sub-label matrices. From Fig. \ref {Fig2b}, we can observe that the ranks of the sub-label matrices are relatively low (the ranks are about 15 and 5). In brief, these observations are consistent with the assumptions of our model. 

\begin{figure}[h]
	\centering
	\subfigure[High-rank]{
		\includegraphics[width=0.227\textwidth]{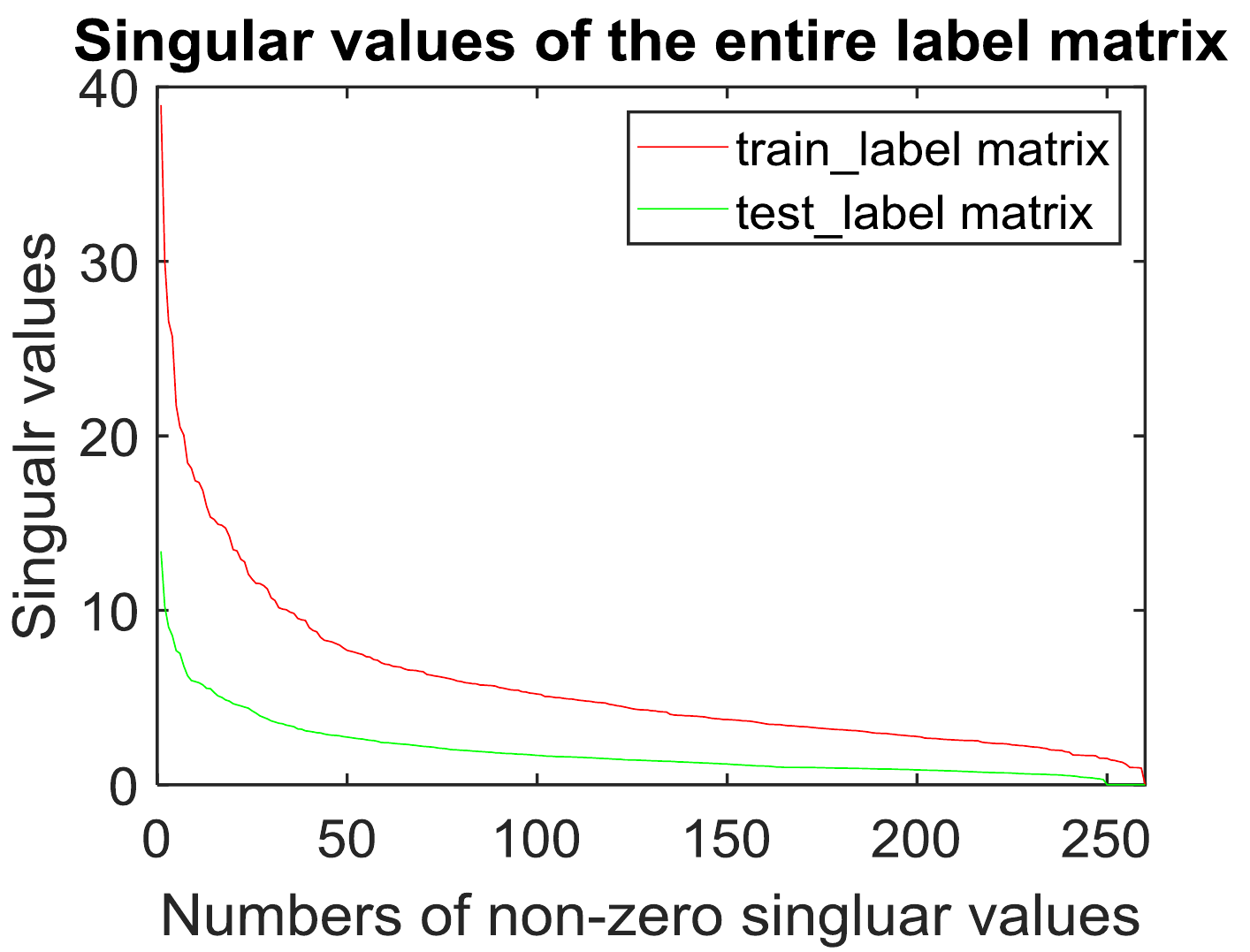}
		\label{Fig2a}
	}
	\subfigure[Low-rank]{
		\includegraphics[width=0.227\textwidth]{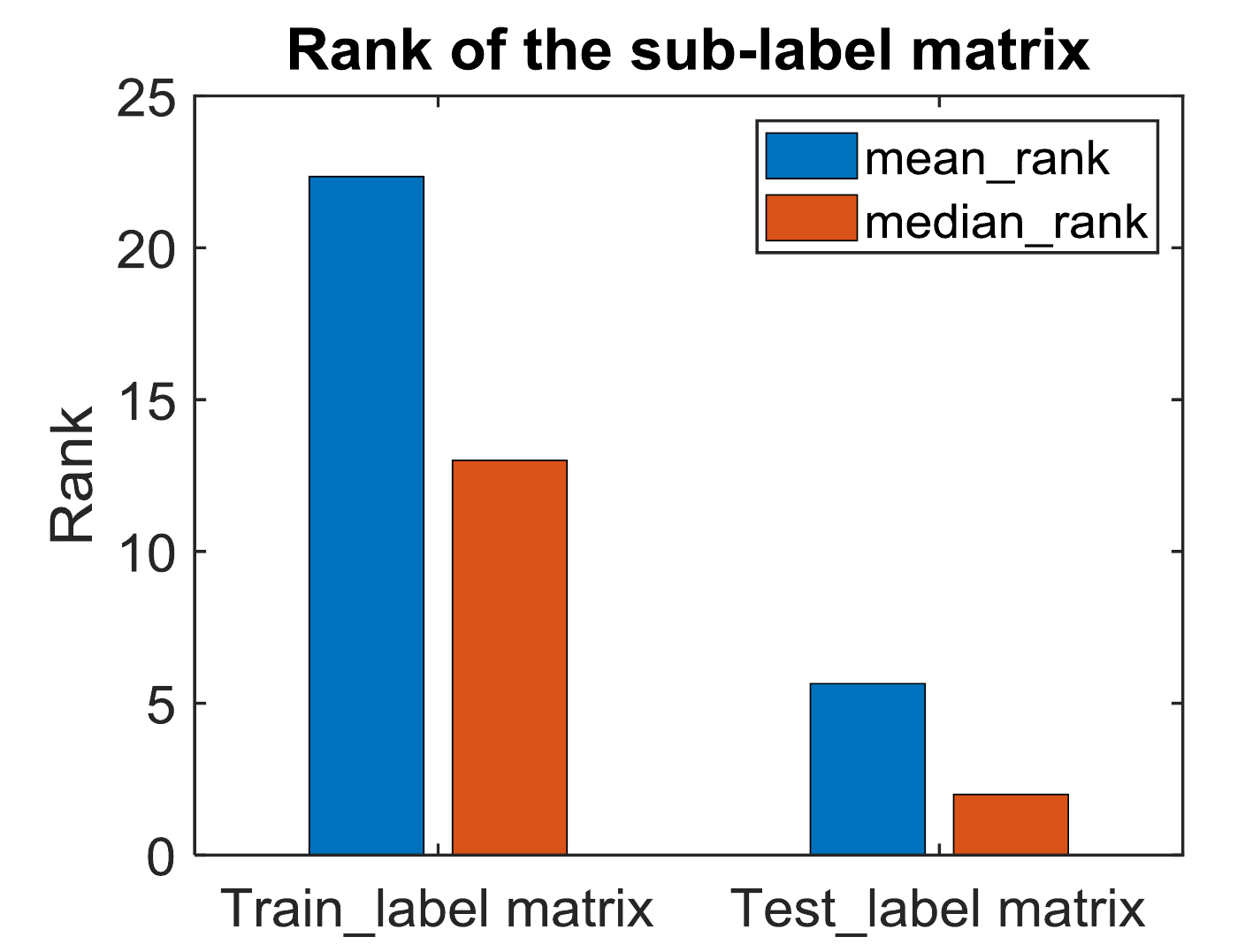}
		\label{Fig2b}
	}
	
	\caption{The high-rankness of the entire label matrix corresponding to all samples and the low-rankness of the sub-label matrices corresponding to samples that share a single label on Corel5k dataset, which has 260 labels.}
\end{figure}

By formulating the above two structures as a single regularizer, we design a final optimization objective for the model establishment. What makes our method most different from existing multi-label learning methods is that the latter almost completely neglect the global high-rank structure, which will be proved in our experiments that such global high rankness plays an indispensable role. In summary, our contributions are fourfold:

(1) To the best of our knowledge, this is the first multi-view multi-label learning work of jointly considering non-alignment of views, incompleteness of views, and missing of labels, which is much more realistic and more challenging. Contrary to other methods, this is the first work that explicitly models the global structure of the multi-label matrix to be high-rank, whose effectiveness has been validated in our experiments.

(2) We utilize the ConCave-Convex Procedure (CCCP) to reduce the objective function as a convex optimization problem, and then provide an efficient Alternating Direction Method of Multipliers (ADMM) algorithm by which a closed form solution of each sub-problem is derived. Besides, the customized ADMM algorithm for NAIM$^3$L has linear computational complexity with respect to the number of samples, which makes it more efficient to handle large scale data.
 
(3) An efficient algorithm enjoying linear time complexity regarding the number of samples is derived as a byproduct to compute the sub-gradient of the trace norm.

(4) Even without view-alignment, our method can still achieve better performance on five real datasets compared to state-of-the-arts with view-alignment.

Compared with other multi-view multi-label learning methods, our model exhibits the following four merits. Firstly, only one hyper-parameter corresponding to the regularization term is introduced in modeling, which makes its optimization greatly easier than existing methods. Secondly, our model is inductive, thus it can be directly applied to predict unseen samples. Thirdly, our model outperforms state-of-the-arts on five real datasets even without view-alignment. Fourthly, our model can also be directly non-linearized to its kernerlized version and cooperate with deep neural network in an end-to-end manner.

The rest of this paper is organized as follows. In Section \ref {secII}, we briefly overview some related work of multi-view multi-label learning. Section \ref {secIII} proposes our method NAIM$^3$L, and an efficient ADMM algorithm is presented in Section \ref {secIV} to solve it. Extensive experimental results and analyses are reported in Section \ref {secV}. Section \ref {secVI} concludes this paper with future research directions.

\section{Related Work} \label{secII}
To date, in the face of the aforesaid three challenges, existing work concerns either only one or two of them. According to the number of challenges addressed, we can roughly divide them into two major categories: methods of addressing one challenge and methods of addressing two challenges, where the former can be divided into three sub-categories: non-aligned multi-view learning, incomplete multi-view learning, and multi-label learning with missing labels. In this section, we overview recent research closely related to ours based on the above taxonomy.

\subsection{Methods of Addressing One Challenge}
\textbf{Non-aligned multi-view learning} deals with the problem that samples in all views are totally unpaired while accompanied with certain constraints from (weakly) supervised information such as must-link (ML) and cannot-link (CL). We will give a formal definition of non-aligned views in subsection \ref{III-A}. To our knowledge, UPMVKSC \cite {r30} is the first and only work that considers the non-aligned multiple views. In UPMVKSC, the authors incorporated the ML and the CL constraints into the kernel spectral clustering to increase the learning performance.  However, this work assumes that views involved are complete and information about the constraints must be given in prior.

\textbf{Incomplete multi-view learning} handles the issue that samples in some views are missing. Recently, some incomplete multi-view learning methods have been proposed. For example, Xu \textit{et al.} \cite {r31} have proposed a method termed MVL-IV to accomplish multi-view learning with incomplete views by assuming that different views should be generated from a shared subspace. Afterwards, Du \textit{et al.} \cite {r32} have modeled the statistical relationships of multi-modality emotional data using multiple modality-specific generative networks with a shared latent space, in which a Gaussian mixture assumption of the shared latent variables is imposed. Lately, Xue \textit{et al.} \cite {r33} have integrated semi-supervised deep matrix factorization, correlated subspace learning, and multi-view label prediction into a unified framework to jointly learn the deep correlated predictive subspace and multi-view shared and private label predictors. Newly, Zhang \textit{et al.}\cite {r34} have presented CPM-Nets to learn the latent multi-view representation through mimicking data transmitting, such that the optimal trade-off between consistence and complementarity across different views can be achieved. In summary, all the above methods assume the views involved are aligned and focus on the task of multi-class learning, which is a special case of multi-label learning when each sample is annotated with only one label.  In addition, there have been abundant literature about incomplete multi-view clustering \cite {r35,r36,r37,r38}, which is not the focus of this paper.

\textbf{Multi-label learning with missing labels} aims to predict the complete labels by giving part of multiple labels of an object. Zhang \textit{et al.} \cite {r39} have proposed a framework to sufficiently leverage the inter-label correlations and the optimal combination of heterogeneous features based on multi-graph Laplacian.  Liu \textit{et al.} \cite {r40} have presented a model called lrMMC that first seeks a low-dimensional common representation of all the views by constraining their common subspace to be low-rank and then utilizes matrix completion for multi-label classification. Zhu \textit{et al.} \cite {r41} have put forward a method termed GLMVML, which extends the GLOCAL \cite {r42} model to its multi-view version by exploiting the global and the local label correlations of all the views and each view simultaneously.
To see the differences of the global and local structures between GLOCAL and our NAIM$^3$L, we briefly describe GLOCAL as follows. Given a dataset $\mathbf{X}$ and its corresponding multi-label matrix $\mathbf{Y}$, GLOCAL exploits the global and local label correlations by the manifold regularization,
\begin{equation} \label{eqa}
\begin{aligned}
\min _{\mathbf{U}, \mathbf{V}, \mathbf{W}}&\left\|\Pi_{\Omega}(\mathbf{Y}-\mathbf{U} \mathbf{V})\right\|_{F}^{2}+\lambda_{1}\left\|\mathbf{V}-\mathbf{W}^{T} \mathbf{X}\right\|_{F}^{2} \\
&+\lambda_{2}(\|\mathbf{U}\|_{F}^{2}+\|\mathbf{V}\|_{F}^{2}+\|\mathbf{W}\|_{F}^{2})+\lambda_{3} \operatorname{tr}\left(\mathbf{F}_{0}^{T} \mathbf{L}_{0} \mathbf{F}_{0}\right) \\
&+\sum_{m=1}^{g} \lambda_{4} \operatorname{tr}\left(\mathbf{F}_{m}^{T} \mathbf{L}_{m} \mathbf{F}_{m}\right)
\end{aligned},
\end{equation}
where $\Pi_{\Omega}$ is a projection operator, $\mathbf{U}\mathbf{V}$ is the low-rank decomposition of $\mathbf{Y}$, and $\mathbf{W}$ is a linear mapping matrix to be learned. $\mathbf{L}_{0}$ and $\mathbf{L}_{m}$ are the Laplacian matrices encoding the global and the local label correlations, respectively. $\mathbf{F}_{0}$ and $\mathbf{F}_{m}$ are the classifier output matrices for all samples and group $m$, respectively.
Although some of them also consider the global and the local structures, the main differences from ours lie in that almost all methods of such a kind assume not only the global and the local manifold structures of the given data but also the low-rankness of the whole label matrix, whereas our method just needs an assumption about the rank of (predictive) label matrix to formulate the global and the local structures, and more importantly, we argue that the whole label matrix should be high-rank opposing to the popular low-rank assumption.

\subsection{Methods of Addressing Two Challenges}
As far as we have known, there are only two methods called iMVWL \cite {r28} and IMVL-IV \cite {r29} that take both incomplete multi-view and missing multi-label into consideration. iMVWL learns a shared subspace from incomplete views by exploiting weak labels and local label correlations, and then trains a predictor in this subspace such that it can capture not only cross-view relationships but also weak-label information of the training samples. We present the formulation of iMVWL to distinguish its low-rank assumption from our high-rank assumption. Specifically, given a multi-view dataset of $n_{v}$ views $\left\{\mathbf{X}_{v}\right\}_{v=1}^{n_{v}}$ and its corresponding multi-label matrix $\mathbf{Y}$, the objective function of iMVWL is formulated as follows:
\begin{equation} \label{eqb}
\begin{aligned}
	\min _{\left\{\mathbf{U}_{v}, \mathbf{V}, \mathbf{W}, \mathbf{S}\right\}} \sum_{v=1}^{n_{v}}\left\|\mathbf{O}^{v} \odot\left(\mathbf{X}_{v}-\mathbf{V} \mathbf{U}_{v}^{T}\right)\right\|_{F}^{2} \\
	+\alpha\|\mathbf{M} \odot(\mathbf{V} \mathbf{W} \mathbf{S}-\mathbf{Y})\|_{F}^{2}+\beta\|\mathbf{S}\|_{*}
\end{aligned},
\end{equation}
where $\odot$ denotes the Hadamard product, $\mathbf{O}^{v}$ and $\mathbf{M}$ are the indicator matrices for the missing views and labels, respectively.  $\mathbf{V}\mathbf{U}_{v}^{T}$ is the non-negative decomposition of $\mathbf{X}_{v}$, $\mathbf{W}$ is the prediction coefficient matrix, and $\mathbf{S}$ is the label correlation matrix. 
Differently, IMVL-IV provides a unified framework for characterizing multiple ingredients including label-specific features, global and local correlations among labels, low-rank assumption of the label matrix, and consistency among the representations of these views. However, these methods require views to be aligned, and involve at least two explicit hyper-parameters in their objectives. More dauntingly, IMVL-IV even contains ten hyper-parameters.

\subsection{Summary}
To summarize, the aforementioned methods have three weaknesses: (1) Most of them assume that the views involved must be aligned, which naturally limits their applicability in practice. (2) Except for MVL-IV, all the methods involve at least two explicit hyper-parameters in their modeling, and some of them even have more extra implicit hyper-parameters making the model selection quite cumbersome. (3) Most of them are created by the non-negative matrix factorization \cite {r43}, resulting in the failure of directly obtaining an inductive learner, thus being suboptimal for predicting unseen samples. In the next section, we will propose a concise yet effective model to overcome the above shortcomings.

\section{The Proposed Method}  \label{secIII}
\subsection{Problem Settings} \label{III-A}
In this subsection, we first give the formal definition of non-aligned views and then present the problem settings of our model in detail.

\textbf{Definition 1.} Given a multi-view multi-label data set $\Omega$,  suppose that $\Omega=\left\{\mathbf{X}^{(i)}\right\}_{i=1}^{V}$ contains $V$ different views, where $\mathbf{X}^{(i)}=\left[\mathbf{x}_{1}^{(i)}, \mathbf{x}_{2}^{(i)}, \cdots, \mathbf{x}_{n}^{(i)}\right] \in \mathbb{R}^{n \times d_{i}}$ is the feature matrix of the $i$-th view, $n$ and $d_i$ are the numbers of samples and the dimensions of features of the $i$-th view, respectively. If samples across all views are totally unpaired, i.e., the $m$-th sample of the $i$-th view $\mathbf{x}_{m}^{(i)}$ and the $m$-th sample of the $j$-th view $\mathbf{x}_{m}^{(j)}$ are distinct samples, for all $m \in\{1,2, \cdots, n\}$, $i, j \in\{1,2, \cdots, V\}$, and $i\ne j$. Then these views are called non-aligned views.

In traditional full-label setting, $\mathbf{Y}^{(i)}=\left[\mathbf{y}_{1}^{(i)}, \mathbf{y}_{2}^{(i)}, \cdots, \mathbf{y}_{n}^{(i)}\right]$ $\in\{-1,1\}^{n \times c}$ is the corresponding label matrix of the $i$-th view and $c$ is the number of multiple labels. $\mathbf{y}_{j k}^{(i)}=1$ $(k=1,2, \cdots, c)$ means the $k$-th label is relevant while $\mathbf{y}_{j k}^{(i)}=-1$ means irrelevant. By considering the missing labels setting, some labels may not be observed, for example, when the $k$-th label of the $j$-th sample in the $i$-th view is missing, $\mathbf{y}_{j k}^{(i)}=0$, and it does not provide any information. Moreover, in the incomplete multi-view scenario, partial views of some samples are missing, correspondingly, the rows of these samples in the feature matrix $\mathbf{X}^{(i)}$ are missing.
\subsection{Problem Formulation}
In this subsection, we focus on the task of predicting the labels of unlabeled test data by learning from non-aligned incomplete multi-view and missing multi-label training data. 

Predicting labels has attracted tremendous interests of researchers in the machine learning community and numerous work has been put forward. Among them, the linear regression \cite {r44} might be the most widely used framework due to its simplicity and effectiveness. Thus, we formulate the prediction as a regression problem. Formally, the loss function can be written as follows,

\begin{equation}\label{Eq1}
\mathcal{L}=\frac{1}{2} \sum_{i=1}^{V}\left\|\mathbf{X}^{(i)} \mathbf{W}^{(i)}-\mathbf{Y}^{(i)}\right\|_{F}^{2}, 
\end{equation}
where $\mathbf{W}^{(i)} \in \mathbb{R}^{d_{i} \times c}$ is the coefficient matrix corresponding to the $i$-th view. Further, to deal with the challenge of missing labels, we introduce an indicator matrix $\mathbf{P}^{(i)}(i=1,2, \cdots, V)$ for each label matrix. Let $\Omega \subseteq\{1,2, \cdots, n\} \times\{1,2, \cdots, c\}$ be the set of indices that observed in the label matrix $\mathbf{Y}^{(i)}$, then the definition of $\mathbf{P}^{(i)}$ is as follows:

\begin{equation}\label{Eq2}
\mathbf{P}_{j k}^{(i)}=\left\{\begin{array}{ll}
1 & \text { if }(j, k) \in \Omega \\
0 & \text {otherwise.}
\end{array}\right. 
\end{equation}

Moreover, views are incomplete in our settings, to alleviate the negative impact arising from the incompleteness of multiple views, we set the rows of $\mathbf{P}^{(i)}$ to zero if the corresponding rows of $\mathbf{X}^{(i)}$ are missing, i.e.,  $\mathbf{P}_{j \bullet}^{(i)}=0$ if the $j$-th sample of the $i$-th view is missing, where $\mathbf{P}_{j \bullet}^{(i)}=0$  denotes the $j$-th row of the indicator matrix $\mathbf{P}^{(i)}$. By introducing $\mathbf{P}^{(i)}$, Eq. \eqref{Eq1} can be rewritten as:

\begin{equation}\label{Eq3}
\mathcal{L}=\frac{1}{2} \sum_{i=1}^{V}\left\|\mathbf{P}^{(i)} \odot\left(\mathbf{X}^{(i)} \mathbf{W}^{(i)}-\mathbf{Y}^{(i)}\right)\right\|_{F}^{2}, 
\end{equation}
where $\odot$ denotes the Hadamard product. Apparently simple as Eq. \eqref{Eq3} seems, it serves three purposes. First, it can be used to predict unlabeled data. Second, inferences of missing labels on training data can be achieved as a byproduct. Third, it can handle both the missing labels and incomplete views.

However, the above loss function neither utilizes multi-view consistence nor exploits multi-label structures. Thus, how to combine these two properties to make our model more discriminative is the main concern in the following. 

Unfortunately, we are confronted with two obstacles when combing aforementioned two properties. One is that non-aligned views make the consensus of multiple views difficult to guarantee. The other is while dealing with non-aligned views, we also need to consider collaborating with the multi-label structures at the same time.

From the observation that although samples among views are not aligned explicitly, they can implicitly be bridged through the common or shared labels. To mitigate the above two obstacles, we align different views in a common label space, in which we characterize the global-local structures of multiple labels. Our motivations are intuitive. First, although the views are not aligned, samples of different views that share the same label should be consistent, hence, views can be aligned by their labels. Second, in real world, samples with similar labels are strongly correlated with each other whereas those with dissimilar labels are weakly correlated, or even uncorrelated. This implies the low rankness of the sub-label matrix of samples sharing the same label and the high rankness of the label matrix of all samples. Finally, the regularizer $\mathcal{R}$ is formulated as:
\begin{equation}\label{Eq4}
\begin{aligned}
\mathcal{R}=& \sum_{k=1}^{c}\left\|[\mathbf{X}_{k}^{(1)} \mathbf{W}^{(1)} ; \mathbf{X}_{k}^{(2)} \mathbf{W}^{(2)}; \cdots ; \mathbf{X}_{k}^{(V)} \mathbf{W}^{(V)}]\right\|_{*} \\
&-\left\|[\mathbf{X}^{(1)} \mathbf{W}^{(1)} ; \mathbf{X}^{(2)} \mathbf{W}^{(2)} ; \cdots ; \mathbf{X}^{(V)} \mathbf{W}^{(V)}] \right\|_{*}
\end{aligned}, 
\end{equation}
where $\|\bullet \|_{*}$ denotes the trace norm, $[\mathbf{A} ; \mathbf{B}]$ is the vertical concatenation of matrices $\mathbf{A}$ and $\mathbf{B}$,  and $\mathbf{X}_{k}^{(i)}$ is the sub-matrix of $\mathbf{X}^{(i)}$ which consists of samples corresponding to the $k$-th label observed in the $i$-th view. Note that, the intersection of  $\mathbf{X}_{k}^{(i)}$ w.r.t. $k$ is non-empty due to the fact that a sample has multiple labels.

By concatenating the samples that share the same single label in all V views (an early fusion strategy), the first term of  Eq. \eqref{Eq4} aims at two purposes. It not only aligns samples of different views in a common label space to ensure consistency but also characterizes the local low-rank structure of each predictive sub-label matrix corresponding to samples that share the same single label. Similarly, the second term aligns diverse views of all samples and depicts the global high-rank structure of multiple labels corresponding to all samples. An intuitive illustration of this global-local structure is shown in Fig. \ref {fig1}. Combining Eq. \eqref{Eq3} and \eqref{Eq4}, the final objective function is formulated as:
\begin{equation}\label{Eq5}
\begin{aligned}
\min _{\mathbf{W}^{(i)}} & \frac{1}{2} \sum_{i=1}^{V}\left\|\mathbf{P}^{(i)} \odot\left(\mathbf{X}^{(i)} \mathbf{W}^{(i)}-\mathbf{Y}^{(i)}\right)\right\|_{F}^{2} \\
&+\lambda\left(\sum_{k=1}^{c}\left\|[\mathbf{X}_{k}^{(1)} \mathbf{W}^{(1)} ; \mathbf{X}_{k}^{(2)} \mathbf{W}^{(2)} ; \cdots ; \mathbf{X}_{k}^{(V)} \mathbf{W}^{(V)}]\right\|_{*}\right.\\
&\left.-\left\|[\mathbf{X}^{(1)} \mathbf{W}^{(1)} ; \mathbf{X}^{(2)} \mathbf{W}^{(2)} ; \cdots ; \mathbf{X}^{(V)} \mathbf{W}^{(V)}]\right\|_{*}\right).
\end{aligned}
\end{equation}

Note that, the two terms of the regularizer $\mathcal{R}$ are designed to jointly describe the global-local structure of multiple labels. More importantly, these two terms work as a whole and either of them is indispensable, which will be validated by ablation study in subsection \ref {V-D}. Thus, we only need one hyper-parameter in our objective function. However, the regularizer $\mathcal{R}$ is the difference of two convex functions, if this term is negative, then we may get trivial solutions when $\lambda$ is large enough. In the following, we will rigorously prove a theorem to claim that $\mathcal{R}$ is non-negative, thus, trivial solutions can be avoided.

\begin{lemma} \label{lemma1} {\rm \cite{r45}}
 Let $\mathbf{A}$ and $\mathbf{B}$ be matrices of the same row dimensions, and $[\mathbf{A}, \mathbf{B}]$ be the concatenation of $\mathbf{A}$ and $\mathbf{B}$, we have $\| [\mathbf{A}, \mathbf{B}]\left\|_{*} \leq\right\| \mathbf{A}\left\|_{*}+\right\| \mathbf{B} \|_{*}$.
\end{lemma}

\begin{theorem} \label{theo1} 
 Let $\mathbf{X}_{k}^{(1)} \mathbf{W}^{(1)}, \mathbf{X}_{k}^{(2)} \mathbf{W}^{(2)}, \cdots, \mathbf{X}_{k}^{(V)} \mathbf{W}^{(V)}$ $(k = 1,2, \cdots,c)$ be matrices with the same column dimension, where $\mathbf{X}_{k}^{(i)}$  is a sub-matrix of $\mathbf{X}^{(i)}(i = 1,2, \cdots,V) $. If (a) $\forall i \in \{1,2,\cdots,V\}$, the vertical concatenation of $\mathbf{X}_{1}^{(i)} \mathbf{W}^{(i)}$ to  $\mathbf{X}_{c}^{(i)} \mathbf{W}^{(i)}$ contains all rows of $\mathbf{X}^{(i)} \mathbf{W}^{(i)}$ and (b) $\forall k,h \in \{1,2,\cdots,c\},k\ne h$, at least one of the intersection between $\mathbf{X}_{k}^{(i)} \mathbf{W}^{(i)}$ and $\mathbf{X}_{h}^{(i)} \mathbf{W}^{(i)}$ is non-empty, then we have 

\begin{equation}\label{Eq6}
\begin{aligned}
& \sum_{k=1}^{c}\left\|[\mathbf{X}_{k}^{(1)} \mathbf{W}^{(1)} ; \mathbf{X}_{k}^{(2)} \mathbf{W}^{(2)} ; \cdots ; \mathbf{X}_{k}^{(V)} \mathbf{W}^{(V)}]\right\| \\
& \geq\left\|[\mathbf{X}^{(1)} \mathbf{W}^{(1)} ; \mathbf{X}^{(2)} \mathbf{W}^{(2)}, \cdots ; \mathbf{X}^{(V)} \mathbf{W}^{(V)}]\right\|_{*}.
\end{aligned}
\end{equation}
\end{theorem}

At the first glance, Theorem \ref{theo1} seems able to be proved directly by extending Lemma \ref {lemma1}, however, in fact, the proof is not that trivial, the detailed proof is shown below. 

\renewcommand{\theequation}{P.\arabic{equation}}
\setcounter{equation}{0}
\begin{proof}  Before proving Theorem \ref{theo1}, firstly, we present the following three propositions.

\begin{equation}     \label{eqA1}
\begin{aligned}
& \left\|\mathbf{A}_{1}\right\|_{*}+\left\|\mathbf{A}_{2}\right\|_{*}+\cdots+\left\|\mathbf{A}_{n}\right\|_{*} \\
\geq & \left\|\left[\mathbf{A}_{1} ; \mathbf{A}_{2} ; \cdots ; \mathbf{A}_{n}\right]\right\|_{*}
\end{aligned}
\end{equation}

\begin{equation}    \label{eqA2}
\begin{aligned}
 & \left\| \left[ \mathbf{A}_{1} ; \mathbf{A}_{2} ; \mathbf{A}_{3} \right] \right\|_{*}=\left\| \left[\mathbf{A}_{1} ; \mathbf{A}_{3} ; \mathbf{A}_{2} \right]\right\|_{*} \\
= & \left\| \left[ \mathbf{A}_{2} ; \mathbf{A}_{1} ; \mathbf{A}_{3} \right] \right\|_{*}=\left\| \left[\mathbf{A}_{2} ; \mathbf{A}_{3} ; \mathbf{A}_{1}\right]\right\|_{*}  \\
= & \left\| \left[ \mathbf{A}_{3} ; \mathbf{A}_{1} ; \mathbf{A}_{2} \right] \right\|_{*}=\left\| \left[\mathbf{A}_{3} ; \mathbf{A}_{2} ; \mathbf{A}_{1}\right]\right\|_{*}
\end{aligned}
\end{equation}

\begin{equation}     \label{eqA3}
\| [\mathbf{A} ; \mathbf{B}] \|_{*} \geq \|\mathbf{A}\|_{*}. 
\end{equation}

\eqref{eqA1} is a generalization of Lemma \ref{lemma1} and can be proved by the following derivation.
\begin{equation}   \nonumber
\begin{aligned}
& \left\|\mathbf{A}_{1}\right\|_{*}+\left\|\mathbf{A}_{2}\right\|_{*}+\cdots+\left\|\mathbf{A}_{n}\right\|_{*} \\
\geq & \left\|\mathbf{A}_{1}+\mathbf{A}_{2}+\cdots+\mathbf{A}_{n}\right\|_{*} (\text {by triangle inequality}) \\
= & \left\|\left[\mathbf{A}_{1} ; \mathbf{0} ; \cdots ; \mathbf{0}\right]+\left[\mathbf{0} ; \mathbf{A}_{2} ; \cdots ; \mathbf{0}\right]+\cdots+\left[\mathbf{0} ; \mathbf{0} ; \cdots ; \mathbf{A}_{n}\right]\right\|_{*}\\
= & \left\|\left[\mathbf{A}_{1} ; \mathbf{A}_{2} ; \cdots ; \mathbf{A}_{n}\right]\right\|_{*}
\end{aligned}
\end{equation}

\eqref{eqA2} is actually the commutative law of the trace norm with respect to the rows (columns). Without loss of generality, we only prove the case of three columns, and it can be directly extended to cases of any number (more than three) of columns.
\begin{equation} \nonumber
\begin{aligned}
& \left\|\mathbf{A}_{1} ; \mathbf{A}_{2} ; \mathbf{A}_{3}\right\|_{*}=t r \sqrt{\left[\mathbf{A}_{1}^{T}, \mathbf{A}_{2}^{T}, \mathbf{A}_{3}^{T}\right]\left[\mathbf{A}_{1} ; \mathbf{A}_{2} ; \mathbf{A}_{3}\right]} \\
= & tr \sqrt{\mathbf{A}_{1}^{T} \mathbf{A}_{1}+\mathbf{A}_{2}^{T} \mathbf{A}_{2}+\mathbf{A}_{3}^{T} \mathbf{A}_{3}} \\
= & tr \sqrt{\mathbf{A}_{1}^{T} \mathbf{A}_{1}+\mathbf{A}_{3}^{T} \mathbf{A}_{3}+\mathbf{A}_{2}^{T} \mathbf{A}_{2}} \\
= & tr \sqrt{\left[\mathbf{A}_{1}^{T}, \mathbf{A}_{3}^{T}, \mathbf{A}_{2}^{T}\right]\left[\mathbf{A}_{1} ; \mathbf{A}_{3} ; \mathbf{A}_{2}\right]}=\left\|\left[\mathbf{A}_{1} ; \mathbf{A}_{3} ; \mathbf{A}_{2}\right]\right\|_{*}
\end{aligned}
\end{equation}
The proof of remaining equations is similar.

\eqref{eqA3} can be easily proved by $\| [\mathbf{A} ; \mathbf{B}] \|_{*}=tr \sqrt{\mathbf{A}^{T} \mathbf{A}+\mathbf{B}^{T} \mathbf{B}} \geq tr \sqrt{\mathbf{A}^{T} \mathbf{A}}=\|\mathbf{A}\|_{*}$. 
  
For simplicity of writing, let $\mathbf{X}_{k}^{(i)} \mathbf{W}^{(i)}=\mathbf{E}_{k}^{(i)}$ and $\mathbf{X}^{(i)} \mathbf{W}^{(i)}=\mathbf{E}^{(i)}$, where $i = 1,2, \cdots,V$ and  $k = 1,2, \cdots,c$ .Then we have,
$\begin{aligned}
& \sum_{k=1}^{c}\left\|[\mathbf{X}_{k}^{(1)} \mathbf{W}^{(1)} ; \mathbf{X}_{k}^{(2)} \mathbf{W}^{(2)} ; \cdots ; \mathbf{X}_{k}^{(V)} \mathbf{W}^{(V)}]\right\|_{*}\\
= &\sum_{k=1}^{c}\left\|[\mathbf{E}_{k}^{(1)} ; \mathbf{E}_{k}^{(2)} ; \cdots ; \mathbf{E}_{k}^{(V)}]\right\|_{*} \\
\geq &\left\|[\mathbf{E}_{1}^{(1)}; \cdots ; \mathbf{E}_{1}^{(V)} ; \mathbf{E}_{2}^{(1)}; \cdots ; \mathbf{E}_{2}^{(V)} ; \cdots ;\mathbf{E}_{c}^{(1)}; \cdots ; \mathbf{E}_{c}^{(V)}]\right\|_{*} \\
= &\left\|[\mathbf{E}_{1}^{(1)} ;  \cdots ; \mathbf{E}_{c}^{(1)} ; \mathbf{E}_{1}^{(2)} ;  \cdots ; \mathbf{E}_{c}^{(2)} ; \cdots ; \mathbf{E}_{1}^{(V)} ;  \cdots ; \mathbf{E}_{c}^{(V)}]\right\|_{*} \\
\geq &\left\|[\mathbf{E}^{(1)} ; \mathbf{E}^{(2)} ; \cdots ; \mathbf{E}^{(V)}]\right\|_{*}\\
= &\left\|[\mathbf{X}^{(1)} \mathbf{W}^{(1)} ; \mathbf{X}^{(2)} \mathbf{W}^{(2)} ; \cdots ; \mathbf{X}^{(V)} \mathbf{W}^{(V)}]\right\|_{*}.
\end{aligned}$

The first inequality holds by \eqref{eqA1}, and the second equality holds by \eqref{eqA2}. In the multi-label setting, the samples corresponding to all the individual labels contain those corresponding to the whole multiple labels, which implies that the condition (a) is satisfied.  Condition (b) is likewise satisfied by the fact that a sample may have multiple labels. Thus, $[\mathbf{E}_{1}^{(i)} ;  \cdots ; \mathbf{E}_{c}^{(i)}]$ can be rewritten as $[ \mathbf{E}_{1}^{(i)} ;  \cdots ; \mathbf{E}_{c}^{(i)} ] = [\mathbf{E}^{(i)}; \mathbf{E}_{in}^{(i)}]$,  where  $\mathbf{E}_{in}^{(i)}$ is the matrix consisting of the intersections of $\mathbf{E}_{k}^{(i)}$ w.r.t. $k$. Then the last inequality holds by \eqref{eqA3}. At this point, the proof is complete. 
\end{proof}
\noindent {\textbf {Remark 2.}} 
A similar method termed as DM2L has been proposed in \cite{r46}. Major differences between this work and DM2L are as follows:

(1) This work generalizes DM2L to the multi-view setting and mainly focuses on the novel and unique challenge posed by non-aligned multiple views, which widely exist in reality while often being neglected. In brief, DM2L can be regarded as a special case of this work.

(2) Albeit the model of DM2L seems similar to ours, the problem addressed in this paper is much more challenging. More importantly, the motivations of these two are different. DM2L claims that the multi-label matrix is low-rank and the negative trace norm term in their model aims at making the DM2L more discriminative, whereas we argue that the multi-label matrix is NOT necessarily low-rank, conversely, high-rank as analyzed before. And it is the latter that we utilize the negative trace norm to directly characterize this property, making our motivations more intuitive and interpretable.

(3) We tailor an efficient optimization method to solve the proposed model. Specifically, we derive an ADMM algorithm with a closed form solution of each sub-problem. The linear computational complexity with respect to the number of samples allows our model to handle large-scale data. However, high efficiency, closed form solution, and the ability to handle large-scale data cannot be guaranteed in DM2L by using traditional convex optimization method. Besides, there are also other differences between NAIM$^3$L and DM2L, for example, we make a thorough computational complexity analysis; the non-negativity of the regularization term needed to be proved is more difficult than that of DM2L and we provide a more concise proof (see details above); we design a more efficient algorithm for computing the sub-gradient of the trace norm.

\section{Optimization} \label{secIV}
\subsection{ADMM  Algorithm}
\renewcommand{\theequation}{\arabic{equation}}
\setcounter{equation}{8}

For the convenience of optimization, we first introduce some notations to simplify the formulas. Let
$
\mathbf{P}=\left[\begin{array}{c}
\mathbf{P}^{(1)} \\
\mathbf{P}^{(2)} \\
\vdots \\
\mathbf{P}^{(V)}
\end{array}\right]
$,  
$
\mathbf{W}=\left[\begin{array}{c}
\mathbf{W}^{(1)} \\
\mathbf{W}^{(2)} \\
\vdots \\
\mathbf{W}^{(V)}
\end{array}\right]
$,
$
\mathbf{X}=\left[\begin{array}{cccc}
\mathbf{X}^{(1)} & \bf 0 &\cdots& \bf 0 \\
\bf 0 & \mathbf{X}^{(2)} &\cdots& \bf 0 \\
\vdots & \vdots & \ddots &\vdots\\ 
\bf 0 & \bf 0 &\cdots& \mathbf{X}^{(V)} 
\end{array}\right]
$,  
$
\mathbf{Y}=\left[\begin{array}{c}
\mathbf{Y}^{(1)} \\
\mathbf{Y}^{(2)} \\
\vdots \\
\mathbf{Y}^{(V)}
\end{array}\right]
$,
and
$
\mathbf{X}_k=\left[\begin{array}{cccc}
\mathbf{X}^{(1)}_k & \bf 0 & \cdots & \bf 0 \\
\bf 0 & \mathbf{X}^{(2)}_k & \cdots & \bf 0 \\
\vdots & \vdots & \ddots & \vdots \\ 
\bf 0 & \bf 0 & \cdots & \mathbf{X}^{(V)}_k 
\end{array}\right]
$,
then Eq. \eqref{Eq5} can be simplified as:

\begin{equation}\label{Eq7}
\begin{aligned}
\min _{\mathbf{w}} &\frac{1}{2}\|\mathbf{P} \odot(\mathbf{X} \mathbf{W}-\mathbf{Y})\|_{F}^{2} \\
&+\lambda\left(\sum_{k=1}^{c}\left\|\mathbf{X}_{k} \mathbf{W}\right\|_{*}-\|\mathbf{X} \mathbf{W}\|_{*}\right),
\end{aligned}
\end{equation}
Eq. \eqref{Eq7} is a DC (Difference of Convex functions) programming, and it can be solved by the ConCave-Convex Procedure (CCCP). Let $ f = J_{cvx}+J_{cav } $ ,
\begin{equation}\label{Eq8}
J_{cvx}=\frac{1}{2}\|\mathbf{P} \odot(\mathbf{X} \mathbf{W}-\mathbf{Y})\|_{F}^{2}+\lambda \sum_{k=1}^{c}\left\|\mathbf{X}_{k} \mathbf{W}\right\|_{*},
\end{equation} 

\begin{equation}\label{Eq9}
J_{cav}=-\lambda\|\mathbf{X} \mathbf{W}\|_{*},
\end{equation}
where $J_{cvx}$ is a convex function and   $J_{cav}$ is a concave function. 

Then by CCCP we have,
\begin{equation}\label{Eq10}
\partial J_{cvx}\left(\mathbf{W}_{t}\right)+\partial J_{cav}\left(\mathbf{W}_{t-1}\right)=0,
\end{equation}
where $\partial J_{cvx}\left(\mathbf{W}_{t}\right)$ is the sub-gradient of $J_{cvx}\left(\mathbf{W}_{t}\right)$  and  $\mathbf{W}_{t}$ is the matrix of the $t$-th iteration. Afterwards, a surrogate objective function  $J$  that satisfies Eq. \eqref{Eq10} can be derived,
\begin{equation}\label{Eq11}
\begin{aligned}
\min _{\mathbf{w}_{t}} J\left(\mathbf{W}_{t}\right)=& \min _{\mathbf{w}_{t}} \frac{1}{2}\left\|\mathbf{P} \odot\left(\mathbf{X} \mathbf{W}_{t}-\mathbf{Y}\right)\right\|_{F}^{2}+\lambda \sum_{k=1}^{c}\left\|\mathbf{X}_{k} \mathbf{W}_{t}\right\|_{*} \\
&-\lambda t r\left[\mathbf{W}_{t}^{T}\left(\partial\left\|\mathbf{X} \mathbf{W}_{t-1}\right\|_{*}\right)\right],
\end{aligned}
\end{equation} 
Eq. \eqref{Eq11} is a convex function w.r.t. $\mathbf{W}_{t}$ and can be solved by off-the-shelf convex optimization toolkit.

However, traditional convex optimization methods often need to search the directions of the gradient, which makes it slow to obtain the optimum. Thus, we tailor an efficient ADMM algorithm and derive the closed form solution of each sub-problem. Specifically, let $\mathbf{Z}_{k}=\mathbf{X}_{k} \mathbf{W}_{t}$, we have the following augmented Lagrangian function,
\begin{equation}\label{Eq12}
\begin{aligned}
\Phi=& \frac{1}{2}\left\|\mathbf{P} \odot\left(\mathbf{X} \mathbf{W}_{t}-\mathbf{Y}\right)\right\|_{F}^{2}+\lambda \sum_{k=1}^{c}\left\|\mathbf{Z}_{k} \right\|_{*} \\
&-\lambda tr \left[\left(\mathbf{X} \mathbf{W}_{t}\right)^{T} \partial\left(\left\|\mathbf{X} \mathbf{W}_{t-1}\right\|_{*}\right)\right] \\
&+\sum_{k=1}^{c} tr \left[\mathbf{\Lambda}_{k}^{T}\left(\mathbf{X}_{k} \mathbf{W}_{t}-\mathbf{Z}_{k}\right)\right]+\frac{\mu}{2} \sum_{k=1}^{c}\left\|\mathbf{Z}_{k}-\mathbf{X}_{k} \mathbf{W}_{t}\right\|_{F}^{2}
\end{aligned},
\end{equation}
where $\mathbf{\Lambda}_{k}$ is the Lagrangian multiplier and $\mu$ is the penalty factor. Note that, $\mu$ is NOT a model hyper-parameter BUT a parameter of the ADMM algorithm that does not need to be adjusted, and it is introduced for the convenience of optimization. Specifically, in Eq. \eqref{Eq12}, with the equipment of the last term, each sub-problem of the ADMM algorithm becomes strongly convex, which guarantees a fast convergence. In the experiments, we will validate that the performance of our model remains unchanged under different $\mu$.

\subsubsection{Sub-problem of $\ \mathbf{W}_{t}$}

With $\mathbf{Z}_{k}$ and $\mathbf{\Lambda}_{k}$ fixed, $\mathbf{W}_{t}$ can be updated by
\begin{equation}\label{Eq13}
\begin{aligned}
\mathbf{W}_{t}&=(\mu \sum_{k=1}^{c} \mathbf{X}_{k}^{T} \mathbf{X}_{k})^{-1}\{\lambda \mathbf{X}^{T} \partial(\|\mathbf{X} \mathbf{W}_{t-1}\|_{*}) \\
&+\sum_{k=1}^{c}[\mathbf{X}_{k}^{T}(\mu \mathbf{Z}_{k}-\mathbf{\Lambda}_{k})]-\mathbf{X}^{T}[\mathbf{P} \odot(\mathbf{X} \mathbf{W}_{t-1}-\mathbf{Y})]\}
\end{aligned}.
\end{equation}

\subsubsection{Sub-problem of $\ \mathbf{Z}_{k}$}

With $\mathbf{W}_{t}$ and $\mathbf{\Lambda}_{k}$ fixed, the objective function of $\mathbf{Z}_{k}$ can be written as:
\begin{equation}\label{Eq14}
\Phi_{\mathbf{Z}_{k}}=\frac{\lambda}{\mu}\left\|\mathbf{Z}_{k}\right\|_{*}+\frac{1}{2}\left\|\mathbf{Z}_{k}-\left(\mathbf{X}_{k} \mathbf{W}_{t}+\frac{\mathbf{\Lambda}_{k}}{\mu}\right)\right\|_{F}^{2}.
\end{equation}
The above problem can be solved by the singular value thresholding algorithm \cite {r47}, and the update rule of $\mathbf{Z}_{k}$ is,
\begin{equation}\label{Eq15}
\begin{aligned}
\mathbf{Z}_{k} &=\Gamma\left(\mathbf{X}_{k} \mathbf{W}_{t}+\frac{\boldsymbol{\Lambda}_{k}}{\mu}\right) \\
&=\mathbf{U} \operatorname{Diag}\left[\left(\sigma_{i}-\frac{\lambda}{\mu}\right)_{+} \right]\mathbf{V}^{T},
\end{aligned}
\end{equation}
where $\Gamma$ is a singular value threshold operator, $\mathbf{U}$ and $\mathbf{V}$ are the matrices of left and right singular vectors of  $\mathbf{X}_{k} \mathbf{W}_{t}+\frac{\mathbf\Lambda_{k}}{\mu}$, $\operatorname{Diag}$ stands for the diagonal matrix, $\sigma_{i}$ is the $i$-th largest singular value and $a_{+}=\max (0, a)$.
\subsubsection{Sub-problem of $\ \mathbf{\Lambda}_{k}$}

With $\mathbf{Z}_{k}$ and $\mathbf{W}_{t}$ fixed, $\mathbf{\Lambda}_{k}$ can be updated by
\begin{equation}\label{Eq16}
\boldsymbol{\Lambda}_{k} \leftarrow \boldsymbol{\Lambda}_{k}+\mu\left(\mathbf{X}_{k} \mathbf{W}_{t}-\mathbf{Z}_{k}\right).
\end{equation}
The entire optimization procedure is summarized in the Algorithm 1. 

\begin{algorithm} \label{Alg1}
\caption{ADMM Algorithm for NAIM$^3$L}
\textbf{Input}: Feature matrix $\mathbf{X}$, observed label matrix $\mathbf{Y}$, indicator matrix $\mathbf{P}$\\
\textbf{Initialization}:  Randomly initialize $\mathbf{W}_0$, $\mathbf{Z}_k = \bf 0$, and  $\mathbf{\Lambda}_{k} = \bf 0$ $(k =1, 2, \cdots, c),  \mu = 5. $\\
\textbf{Output}: $\mathbf{W}$
\begin{algorithmic}[1] 
\STATE Let $t=0$.
\WHILE{not converge}
\STATE $t =  t + 1$.
\STATE Update $\mathbf{W}_t$ by Eq. \eqref{Eq13}. 
\STATE Update $\mathbf{Z}_k$ by Eq. \eqref{Eq15}.
\STATE Update $\mathbf{\Lambda}_{k}$ by Eq. \eqref{Eq16}.
\ENDWHILE
\STATE \textbf{return} $\mathbf{W}$
\end{algorithmic}
\end{algorithm}

\subsection{An Efficient Algorithm for Computing $\partial\|\bullet \|_{*}$}
Let $ \mathbf {A} \in \mathbb{R}^{n \times c}$ be an arbitrary matrix and $\mathbf{U} \boldsymbol{\Sigma} \mathbf{V}^{T}$ be its singular value decomposition (SVD), where $\mathbf {U}$ and $\mathbf {V}$ are the   matrices of left and right singular vectors, respectively. It is well known \cite {r47,r48} that the sub-gradient of the trace norm can be computed by
$\partial\|\mathbf{A}\|_{*}=\left\{\mathbf{U} \mathbf{V}^{T}+\mathbf{Q}| \quad \mathbf{Q} \in \mathbb{R}^{n \times c}, \mathbf{U}^{T} \mathbf{Q}=\mathbf{0},\mathbf{Q} \mathbf{V}=\mathbf{0}, \|\mathbf{Q}\|_{2} \leq 1\right\}$, where $\|\bullet \|_{2}$ is the spectral norm.

For simplicity, let $\mathbf{Q}= \mathbf{0}$ , then $\partial\|\mathbf{A}\|_{*}$ can be computed by $\partial\|\mathbf{A}\|_{*}=\mathbf{U} \mathbf{V}^{T}$. 
In the following, we will derive a theorem and then design an efficient algorithm to compute $\partial\|\mathbf{A}\|_{*}$ based on the theorem.

\begin{theorem} \label{theo2} 
Let $\mathbf {A} = \mathbf{U} \boldsymbol{\Sigma} \mathbf{V}^{T}$ be the SVD of matrix $\mathbf{A}$, then $\mathbf{A}^{T} \mathbf{A}=\mathbf{V} \mathbf{\Sigma}^2 \mathbf{V}^{T} = \mathbf{V} \mathbf{S} \mathbf{V}^{T}$ is the eigenvalue decomposition of the matrix $\mathbf{A}^{T} \mathbf{A}$, and $\partial\|\mathbf{A}\|_{*}=\mathbf{U} \mathbf{V}^{T} = \mathbf{A} \mathbf{V} \mathbf{S}^{-\frac{1}{2}} \mathbf{V}^{T} $.
\end{theorem} 
\begin{proof}
$\mathbf{A}=\mathbf{U} \Sigma \mathbf{V}^{T}$, then $\mathbf{A}^{T} \mathbf{A}=\mathbf{V} \mathbf{\Sigma} \mathbf{U}^{T} \mathbf{U} \Sigma \mathbf{V}^{T} =\mathbf{V} \mathbf{\Sigma}^2 \mathbf{V}^{T}$.
Let $\mathbf{S} = \mathbf{\Sigma}^2$, then $\mathbf{A}^{T} \mathbf{A}=\mathbf{V} \mathbf{S} \mathbf{V}^{T}$ is the eigenvalue decomposition of the matrix $\mathbf{A}^{T} \mathbf{A}$.

Rewrite $\mathbf{A}^{T} \mathbf{A} =\mathbf{V} \mathbf{\Sigma}^2 \mathbf{V}^{T} =\mathbf{V} \mathbf{\Sigma} \mathbf{V}^{T} \mathbf{V} \mathbf{\Sigma} \mathbf{V}^{T}$,
then $\left(\mathbf{A}^{T} \mathbf{A}\right)^{-\frac{1}{2}}=\left(\mathbf{V} \mathbf{\Sigma} \mathbf{V}^{T}\right)^{-1}=\mathbf{V} \mathbf{\Sigma}^{-1} \mathbf{V}^{T} = \mathbf{V} \mathbf{S}^{-\frac{1}{2}} \mathbf{V}^{T}$ and $ \mathbf{A}\left(\mathbf{A}^{T} \mathbf{A}\right)^{-\frac{1}{2}}=\mathbf{U} \Sigma \mathbf{V}^{T} \mathbf{V} \mathbf{\Sigma}^{-1} \mathbf{V}^{T} =\mathbf{U} \mathbf{V}^{T}$.

Finally, $\partial\|\mathbf{A}\|_{*}=\mathbf{U} \mathbf{V}^{T} = \mathbf{A} \mathbf{V} \mathbf{S}^{-\frac{1}{2}} \mathbf{V}^{T} $.
\end{proof}

According to Theorem \ref{theo2}, we can now design an efficient algorithm for computing $\partial\|\mathbf{A} \|_{*}$ by transforming the SVD of an $n \times c$ matrix into the eigenvalue decomposition of a $c \times c$ matrix.

If the full SVD is adopted to compute $\partial\|\mathbf{A} \|_{*}$, the whole time complexity is $\mathcal{O}\left(n c^{2} + c n^{2} \right) + \mathcal{O}\left( mrc\right)$, where $r$ is the rank of matrix $\mathbf{A}$. However, the time complexity of Algorithm 2 is $\mathcal{O}\left(n c^{2} \right) + \mathcal{O}\left( c^{3} \right)+ \mathcal{O}\left( n c^{2} \right)$. Generally, in the multi-label setting, the number of samples is much larger than that of multiple labels, i.e., $n\gg c$, so the time complexity of adopting full SVD is $\mathcal{O}\left( c n^{2} \right)$ whereas the time complexity of Algorithm 2 is $\mathcal{O}\left( n c^{2} \right)$, which is much more efficient. Algorithm 2 can be more efficient if a more sophisticated algorithm is elaborately customized. However, it is not the main focus of this work but just a byproduct, which itself also has an interest of independence to great extent. We argue that Algorithm 2 can be implemented quite easily with a few lines of codes in Matlab. More importantly, it is effective enough for us to handle large-scale datasets by the fact that the computation complexity with respect to the number of samples reduces from quadratic to linear. Regarding the ''efficient algorithm'' mentioned here, what we aim to emphasize is that Algorithm 2 can handle large-scale datasets, instead of that it can defeat the state-of-the-arts SVD algorithms. Experiments in section \ref{5.9} will validate the efficiency of Algorithm 2.

\begin{algorithm} [t] \label{Alg2}  
\caption{An Efficient Algorithm for Computing $\partial\|\bullet \|_{*}$}
\textbf{Input}: A matrix $ \mathbf {A} \in \mathbb{R}^{n \times c}$, \\
\textbf{Output}: $\partial\|\mathbf {A}||_{*}$.
\begin{algorithmic}[1] 
\IF     {$ n \ge c $}
\STATE  $\mathbf{B}=\mathbf{A}^{T} \mathbf{A}$.
\ELSE
\STATE  $\mathbf{B}=\mathbf{A} \mathbf{A}^{T}$.
\ENDIF
\STATE Eigenvalue decomposition of $\mathbf{B}$, $\mathbf{B}= \mathbf{V} \mathbf{S} \mathbf{V}^{T}$.
\STATE $\partial\|\mathbf {A}||_{*}= \mathbf{A} \mathbf{V} \mathbf{S}^{-\frac{1}{2}} \mathbf{V}^{T}$. 
\STATE \textbf{return} $\partial\|\mathbf {A}||_{*}$
\end{algorithmic}
\end{algorithm}

\subsection{Convergence and Complexity Analysis} \label{IV-C}
\emph{\textbf{Convergence Analysis}}. Before giving the convergence analysis of Algorithm 1, we introduce the following lemma.
\begin{lemma} \label{lemma2}  {\rm \cite{r49}}
 Consider an energy function $J(x)$ of form $J(x) = J_{cvx}(x) + J_{cav}(x)$, where $J_{cvx}(x)$, $J_{cav}(x)$ are convex and concave functions of $x$, respectively. Then the discrete iterative CCCP algorithm $ x^{t}\longmapsto x^{t+1}$ given by
\begin{equation} \nonumber
\nabla J_{cvx}(x^{t+1}) = - \nabla J_{cav}(x^{t})
\end{equation}
is guaranteed to monotonically decrease the energy $J(x)$ as a function of time and
hence to converge to a minimum or saddle point of $J(x)$.
\end{lemma}

According to Lemma \ref{lemma2} and Eq. \eqref{Eq10}, we can derive that $J_{cvx}\left(\mathbf{W}_{t}\right)+J_{cav}\left(\mathbf{W}_{t}\right) \leq J_{cvx}\left(\mathbf{W}_{t-1}\right)+J_{cav}\left(\mathbf{W}_{t-1}\right)$, which means the objective function $f$ is guaranteed to monotonically decrease. Moreover, according to Theorem \ref{theo1}, it is easy to validate that the objective function $f$ is lower bounded by 0. Thus, $f$ is guaranteed to converge by the above two facts. Besides, in the ADMM algorithm, the surrogate objective function  $J$  is strongly convex, which guarantees a global optimum of each sub-problem. Therefore, Algorithm 1 is guaranteed to converge.  

\emph{\textbf {Complexity Analysis}}. The time complexity of NAIM$^3$L is dominated by matrix multiplication and inverse operations. In each iteration, the complexity of updating $\mathbf{W}_t$ in Eq. \eqref{Eq13} is $\mathcal{O}\left[V(n d_{\max }^{2} c+d_{\max }^{3}+n d_{\max } c+n c^{2}+n d_{\max } c^{2}\right)]$ and the complexity of updating $\mathbf{Z}_k$ in Eq. \eqref{Eq15} is $\mathcal{O}\left[V(n c^{3} + n d_{\max } c^{2}\right)]$. The update of  $\mathbf{\Lambda}_{k}$ in Eq. \eqref{Eq16} costs $\mathcal{O}\left(Vn d_{\mathrm{max}} c^{2}\right)$. Generally, $n>d_{\max }$, $n \gg c$  and $d_{\max } > c$, so the total complexity of NAIM$^3$L is $\mathcal{O}\left(t Vn d_{\max }^{2} c\right)$, where $t$ is number of iterations, $V$ is the number of views, $n$ is the number of samples, $d_{\max }$ is the maximum dimension of the features and $c$ is the number of multiple labels. Therefore, NAIM$^3$L has linear computational complexity with respect to the number of samples, which enables it more efficiently to handle large scale data.

\section{Experiments}     \label{secV}
\subsection{Experimental Settings}
\textbf {Datasets:} Five real datasets Corel5k, Espgame, IAPRTC12, Mirflickr, and Pascal07 are used in our experiments. They are available at website\footnote{\url{ http://lear.inrialpes.fr/people/guillaumin/data.php}} \cite{r50}. For fairness, we use exactly the same settings as in iMVWL \cite {r28}. Specifically, each dataset involves six views: HUE, SIFT, GIST, HSV, RGB, and LAB. For each dataset, we randomly sample 70\% of the data for training and use the remaining 30\% data for testing (unlabeled data). Furthermore, in the incomplete multi-view setting, we randomly remove $\alpha$ samples in each view while ensuring each sample appears in at least one view; in the missing labels setting, for each label, we randomly remove $\beta$ positive and negative tags of the training samples; in the non-aligned multi-view setting, samples of all views are randomly arranged and totally unpaired as defined in subsection \ref{III-A}. During the process of training and learning, the alignment information of the samples (orders of the samples) is completely unknown to us. The statistics of the datasets are summarized in Table \ref{tab1}. From this table, we can find that the label matrices are of full column rank or high rank.
\begin{table}[h]
\centering
\caption{Statistics of the Datasets. {\upshape n} and {\upshape c} are the Numbers of Samples and Multiple Labels in Each Dataset,  and \#{\upshape avg} is the Average Number of Relevant Labels in Each Sample. {\upshape train\_rank} and {\upshape test\_rank} are the Ranks of the Label Matrices of Training Set and Testing Set.} 
\label{tab1}
\begin{tabular}{@{}cccccc@{}} 
\toprule
datasets  & n     & c   & \#avg   & train\_rank  & test\_rank    \\ \midrule
Corel5k   & 4999  & 260 & 3.396   &     259     &   249        \\
Espgame   & 20770 & 268 & 4.686   &     268     &   268        \\
IAPRTC12  & 19627 & 291 & 5.719   &     291     &   291        \\
Mirflickr & 25000 & 38  & 4.716   &      38     &    38        \\
Pascal07  & 9963  & 20  & 1.465   &      20     &    20        \\ \bottomrule
\end{tabular}
\end{table}

\textbf {Compared methods:} In the experiments, NAIM$^3$L is compared with five state-of-the-art methods: iMSF \cite {r51}, LabelMe \cite {r39}, MVL-IV \cite {r31}, lrMMC \cite {r40}, and iMVWL \footnote{We sincerely thank the authors of iMVWL for providing the codes, however, their codes did not fix the random seeds, thus results in their original paper cannot be reproduced. We run their codes by using the optimal hyper-parameters suggested in their paper and fix the same random seeds as in our codes. Note that, half of the results of iMVWL are better than those reported in their original paper.} \cite {r28}. iMSF is a multi-class learning method, and we extend it for multi-label classification by training multiple classifiers (one for each label). IMVL-IV \cite {r29} is an incomplete multi-view and missing multi-label learning method, however, it contains ten hyper-parameters in its model, making it very difficult to tune for the optimum. For fairness, we omit it. Besides, there are also some deep neural network (DNN) based methods concerning the complete multi-view multi-label learning \cite{r3, r13}. In this work, we mainly focus on the novel problem of the non-aligned views with missing multiple labels and provide a simple yet effective solution. Considering that our model can directly cooperate with DNN, comparisons with these methods will be conducted in our future work.  It is worth noting that all the compared methods cannot deal with non-aligned views, thus in the experiments, they are all implemented only with the missing labels and incomplete multi-view settings, whereas our NAIM$^3$L is conducted in the settings with all the three challenges. Optimal parameters for the competitive methods are selected as suggested in the corresponding papers. All experiments are repeated ten times, and both the mean and standard deviation are reported.

\textbf {Evaluation metrics:} Similar to iMVWL, four widely used multi-label evaluation metrics are adopted for performance evaluations, i.e., Ranking Loss (RL), Average Precision (AP), Hamming Loss (HL), and adapted Area Under ROC Curve (AUC). A formal definition of the first three metrics can be found in \cite {r9}. The adapted AUC is suggested in \cite {r52}. For consistency, in our experiments, we report 1-RL and 1-HL instead of RL and HL. Thus, the larger the values of all four metrics are, the better the performance is.

\subsection{Main Experimental Results} \label{5.2}
\subsubsection{Experiments under Incomplete Views, Missing Labels (and non-Aligned Views)}
In this subsection, we show the results under the settings of incomplete views, missing labels, and non-aligned views. To the best of our knowledge, existing methods cannot work under the non-aligned views setting. For this reason, the compared methods are all implemented only with the incomplete views and missing labels settings. However, our NAIM$^3$L is conducted in all three settings, which means that in the experiments, information available for NAIM$^3$L is much less than other methods.
Table \ref{tab2} shows the results compared with other methods under the setting of 50\% incomplete views and 50\% missing positive and negative labels. From this table, we can see that, even though without view alignment information, NAIM$^3$L still outperforms all the compared methods with view alignment information on five datasets. Note that, in these experiments, other methods are under the aligned-views setting, whereas ours are not. Nonetheless, with less available information and without view completion, NAIM$^3$L still achieves better performance. This can be attributed to the joint consideration of the local low-rank and global high-rank structures within multiple labels while the latter is almost completely neglected in other methods. In subsection \ref{V-D}, we will conduct extensive experiments to validate the importance of the global high-rank structure of multiple labels.


\begin{table*}[h]
\centering
\caption{Results on all Five Datasets with the Ratio of Incomplete Multi-view $\alpha = 50$\% and the Ratio of Missing Multi-label $\beta = 50$\%. Values in Parentheses Represent the Standard Deviation, and all the Values are Displayed as Percentages.}
\label{tab2}
\begin{tabular}{@{}cccccccc@{}}
\toprule
dataset                    & metrics  & lrMMC       & MVL-IV      & LabelMe     & iMSF        & iMVWL       & NAIM$^3$L      \\ \midrule
\multirow{4}{*}{Corel5k}   & 1-HL(\%) & 95.40(0.00) & 95.40(0.00) & 94.60(0.00) & 94.30(0.00) & 97.84(0.02) & \textbf{98.70}(0.01) \\
                           & 1-RL(\%) & 76.20(0.20) & 75.60(0.10) & 63.80(0.30) & 70.90(0.50) & 86.50(0.33) & \textbf{87.84}(0.21) \\
                           & AP(\%)   & 24.00(0.20) & 24.00(0.10) & 20.40(0.20) & 18.90(0.20) & 28.31(0.72) & \textbf{30.88}(0.35) \\
                           & AUC(\%)  & 76.30(0.20) & 76.20(0.10) & 71.50(0.10) & 66.30(0.50) & 86.82(0.32) & \textbf{88.13}(0.20) \\ \midrule
\multirow{4}{*}{Pascal07}  & 1-HL(\%) & 88.20(0.00) & 88.30(0.00) & 83.70(0.00) & 83.60(0.00) & 88.23(0.38) & \textbf{92.84}(0.05) \\ 
                           & 1-RL(\%) & 69.80(0.30) & 70.20(0.10) & 64.30(0.40) & 56.80(0.00) & 73.66(0.93) & \textbf{78.30}(0.12) \\
                           & AP(\%)   & 42.50(0.30) & 43.30(0.20) & 35.80(0.30) & 32.50(0.00) & 44.08(1.74) & \textbf{48.78}(0.32) \\
                           & AUC(\%)  & 72.80(0.20) & 73.00(0.10) & 68.60(0.50) & 62.00(0.10) & 76.72(1.20) & \textbf{81.09}(0.12) \\ \midrule
\multirow{4}{*}{ESPGame}   & 1-HL(\%) & 97.00(0.00) & 97.00(0.00) & 96.70(0.00) & 96.40(0.00) & 97.19(0.01) & \textbf{98.26}(0.01) \\
                           & 1-RL(\%) & 77.70(0.10) & 77.80(0.00) & 68.30(0.20) & 72.20(0.20) & 80.72(0.14) & \textbf{81.81}(0.16) \\
                           & AP(\%)   & 18.80(0.00) & 18.90(0.00) & 13.20(0.00) & 10.80(0.00) & 24.19(0.34) & \textbf{24.57}(0.17) \\
                           & AUC(\%)  & 78.30(0.10) & 78.40(0.00) & 73.40(0.10) & 67.40(0.30) & 81.29(0.15) & \textbf{82.36}(0.16) \\ \midrule
\multirow{4}{*}{IAPRTC12}  & 1-HL(\%) & 96.70(0.00) & 96.70(0.00) & 96.30(0.00) & 96.00(0.00) & 96.85(0.02) & \textbf{98.05}(0.01) \\
                           & 1-RL(\%) & 80.10(0.00) & 79.90(0.10) & 72.50(0.00) & 63.10(0.00) & 83.30(0.27) & \textbf{84.78}(0.11) \\
                           & AP(\%)   & 19.70(0.00) & 19.80(0.00) & 14.10(0.00) & 10.10(0.00) & 23.54(0.39) & \textbf{26.10}(0.13) \\
                           & AUC(\%)  & 80.50(0.00) & 80.40(0.10) & 74.60(0.00) & 66.50(0.10) & 83.55(0.22) & \textbf{84.96}(0.11) \\ \midrule
\multirow{4}{*}{Mirflickr} & 1-HL(\%) & 83.90(0.00) & 83.90(0.00) & 77.80(0.00) & 77.50(0.00) & 83.98(0.28) & \textbf{88.15}(0.07) \\
                           & 1-RL(\%) & 80.20(0.10) & 80.80(0.00) & 77.10(0.10) & 64.10(0.00) & 80.60(1.11) & \textbf{84.40}(0.09) \\
                           & AP(\%)   & 44.10(0.10) & 44.90(0.00) & 37.50(0.00) & 32.30(0.00) & 49.48(1.24) & \textbf{55.08}(0.18) \\
                           & AUC(\%)  & 80.60(0.10) & 80.70(0.00) & 76.10(0.00) & 71.50(0.10) & 79.44(1.46) & \textbf{83.71}(0.06) \\
\bottomrule
\end{tabular}
\end{table*}

\subsubsection{Experiments under Incomplete Views} 
To further validate the effectiveness of our NAIM$^3$L, we conduct experiments under the settings of incomplete views, full labels, and aligned views. iMVWL \cite {r28} is a method that can deal with both the incomplete views and missing labels. In this subsection, the settings are 50\% incomplete views, full labels, and aligned views. iMVWL-V and NAIM$^3$L-V are the corresponding methods dealing with incomplete views. As our method is customized for the non-aligned views, NAIM$^3$L-V is still conducted under the non-aligned views setting. From Table \ref{tab3}, we can see that NAIM$^3$L-V outperforms iMVWL-V. Besides, a weird phenomenon is found when comparing Tables \ref{tab2} and \ref{tab3}, i.e., the performance of iMVWL under the full labels setting is worse than that under the missing labels setting on the Pascal07, IAPRTC12, and Mirflickr datasets (the worse performance is denoted by down-arrows in Table \ref{tab3}). The reason for this counter-intuitive phenomenon will be explained and analyzed in subsection \ref {5.2.4}.

\begin{table}[h] 	
	\centering
	\caption{Results on all Five Datasets with Full Labels and the Ratio of Incomplete Multi-view $\alpha = 50$\%. Values in Parentheses Represent the Standard Deviation, and all the Values are Displayed as Percentages. Down-Arrows Denote that Performance of iMVWL under Full Label Setting is Worse than that under 50\% Missing Label Setting.} \label{tab3}
	\begin{tabular}{cclc}
		\toprule
		datasets                   & metrics  & iMVWL-V      & NAIM$^3$L-V          \\ \midrule
		\multirow{4}{*}{Corel5k}   & 1-HL(\%) & 97.85(0.03)  & \textbf{98.70}(0.01) \\
		                           & 1-RL(\%) & 87.00(0.27)  & \textbf{88.52}(0.25) \\
		                           & AP(\%)   & 28.90(0.89)  & \textbf{31.72}(0.32) \\
		                           & AUC(\%)  & 87.30(0.26)  & \textbf{88.81}(0.23) \\ \midrule
		\multirow{4}{*}{Pascal07}  & 1-HL(\%) & 88.19(0.28)$\downarrow$   & \textbf{92.87}(0.05) \\
		                           & 1-RL(\%) & 73.74(0.53)  & \textbf{79.07}(0.08) \\
		                           & AP(\%)   & 43.54(0.87)$\downarrow$   & \textbf{49.25}(0.24) \\
		                           & AUC(\%)  & 76.80(0.59)  & \textbf{81.81}(0.10) \\ \midrule
		\multirow{4}{*}{ESPGame}   & 1-HL(\%) & 97.19(0.01)  & \textbf{98.27}(0.01) \\
		                           & 1-RL(\%) & 80.96(0.15)  & \textbf{82.22}(0.15) \\
		                           & AP(\%)   & 24.46(0.40)  & \textbf{24.89}(0.17) \\
		                           & AUC(\%)  & 81.55(0.19)  & \textbf{82.77}(0.16) \\ \midrule
		\multirow{4}{*}{IAPRTC12}  & 1-HL(\%) & 96.84(0.02)$\downarrow$   & \textbf{98.05}(0.01) \\
		                           & 1-RL(\%) & 83.23(0.26)$\downarrow$   & \textbf{85.14}(0.09) \\
		                           & AP(\%)   & 23.40(0.46)$\downarrow$   & \textbf{26.45}(0.14) \\
		                           & AUC(\%)  & 83.50(0.22)$\downarrow$   & \textbf{85.29}(0.10) \\ \midrule
		\multirow{4}{*}{Mirflickr} & 1-HL(\%) & 83.89(0.20)$\downarrow$   & \textbf{88.17}(0.07) \\
		                           & 1-RL(\%) & 80.59(0.72)$\downarrow$   & \textbf{84.55}(0.08) \\
		                           & AP(\%)   & 48.95(0.97)$\downarrow$   & \textbf{55.30}(0.16) \\
		                           & AUC(\%)  & 79.69(1.15)  & \textbf{83.84}(0.05) \\ \bottomrule
	\end{tabular}
\end{table}

\subsubsection{Experiments under Missing Labels}
Similarly, experimental settings in this subsection are 50\% of missing labels, complete views, and aligned views. iMVWL-L and NAIM$^3$L-L are the corresponding methods dealing with missing labels. Still, NAIM$^3$L-L is conducted under the non-aligned views. Not surprisingly, NAIM$^3$L-L once again outperforms iMVWL-L. When comparing Tables \ref{tab2} and \ref{tab4}, the weird phenomenon that appeared in Table \ref{tab3} disappeared, and we will analyze these results in the next subsection.

\begin{table}[h]
	\centering
	\caption{Results on all Five Datasets with Complete views and the Ratio of Missing Multi-Label $\beta = 50$\%. Values in Parentheses Represent the Standard Deviation, and all the Values are Displayed as Percentages.}\label{tab4} 	
	\begin{tabular}{cccc}
		\toprule
		datasets                   & metrics  & iMVWL-L      & NAIM$^3$L-L          \\ \midrule
		\multirow{4}{*}{Corel5k}   & 1-HL(\%) & 97.91(0.01)  & \textbf{98.70}(0.01) \\
		                           & 1-RL(\%) & 88.00(0.31)  & \textbf{88.55}(0.28) \\
		                           & AP(\%)   & 30.77(0.58)  & \textbf{31.93}(0.39) \\
		                           & AUC(\%)  & 88.32(0.32)  & \textbf{88.84}(0.26) \\ \midrule
		\multirow{4}{*}{Pascal07}  & 1-HL(\%) & 88.79(0.12)  & \textbf{92.87}(0.06) \\
		                           & 1-RL(\%) & 76.30(0.59)  & \textbf{79.04}(0.15) \\
		                           & AP(\%)   & 46.95(0.55)  & \textbf{49.30}(0.31) \\
		                           & AUC(\%)  & 79.24(0.61)  & \textbf{81.79}(0.13) \\ \midrule
		\multirow{4}{*}{ESPGame}   & 1-HL(\%) & 97.23(0.01)  & \textbf{98.27}(0.01) \\
		                           & 1-RL(\%) & 81.54(0.22)  & \textbf{82.20}(0.14) \\
		                           & AP(\%)   & 25.80(0.34)  & \textbf{24.90}(0.17) \\
		                           & AUC(\%)  & 82.03(0.19)  & \textbf{82.74}(0.14) \\ \midrule
		\multirow{4}{*}{IAPRTC12}  & 1-HL(\%) & 96.90(0.01)  & \textbf{98.05}(0.01) \\
		                           & 1-RL(\%) & 84.14(0.17)  & \textbf{85.11}(0.10) \\
		                           & AP(\%)   & 25.00(0.18)  & \textbf{26.47}(0.13) \\
		                           & AUC(\%)  & 84.17(0.15)  & \textbf{85.27}(0.10) \\ \midrule
		\multirow{4}{*}{Mirflickr} & 1-HL(\%) & 84.25(0.58)  & \textbf{88.18}(0.07) \\
		                           & 1-RL(\%) & 81.22(1.97)  & \textbf{84.57}(0.08) \\
		                           & AP(\%)   & 50.23(2.03)  & \textbf{55.36}(0.14) \\
		                           & AUC(\%)  & 79.48(2.91)  & \textbf{83.86}(0.06) \\ \bottomrule
	\end{tabular}
\end{table}

\subsubsection{Analysis and Summary} \label{5.2.4}
In this subsection, we make an explanation of the counter-intuitive phenomenon in Table \ref{tab3} and summarize the experimental results of subsection \ref{5.2}. By comparing Tables \ref{tab2} and \ref{tab4}, when the missing label ratio is fixed and the incomplete view ratio changes, results of iMVWL are intuitively consistent. However, from Tables \ref{tab2} and \ref{tab3}, when the incomplete view ratio is fixed and the missing label ratio changes, results of iMVWL are counter-intuitive. These results indicate that iMVWL deals with incomplete view relatively appropriate whereas inappropriate when dealing with missing multi-labels. It can be attributed to that iMVWL makes an improper low-rank assumption about the entire multi-label matrix (see Eq. (\ref{eqb})), which violates the reality. Differently, results of NAIM$^3$L are all intuitively consistent since the high-rank assumption about the entire multi-label matrix in NAIM$^3$L is supported by observations of real datasets and thus more appropriate than that of iMVWL.

To summarize, the results in Tables \ref{tab2}, \ref{tab3}, and \ref{tab4} indicate that our NAIM$^3$L consistently performs the best whether dealing with one challenge, two, or three, and the high-rank assumption indeed makes sense in all these settings.

\subsection{On the High/Low Rank Validation of the Predicted Multi-Label Matrices}
To justify the rationality of the assumptions in NAIM$^3$L, (i.e., the high-rankness of the entire label matrix and the low-rankness of the sub-label matrices), we conduct experiments on the predicted multi-label matrices by using the learned $\mathbf{W}$ in Algorithm 1. First, we validate the high-rankness of the predicted entire multi-label matrix by showing that the label matrices are of full column rank in Table \ref{tab5}. Second, we plot the nuclear norm value of the predicted entire multi-label matrix in Fig. \ref {Fig3}. As the number of multiple labels is too big to show the low-rankness of each sub-label matrix in a single figure, we show the mean value and the median value of the nuclear norm of the predicted sub-label matrices to validate the low-rank assumption. From Fig. \ref {Fig3}, we can find that the mean value and the median value are relatively small, which means the predicted sub-label matrices are of low-rank for the reason that the nuclear norm of a matrix is an upper bound of its rank. In a word, the ranks of the predicted multi-label matrices are consistent with the assumptions of our model.

\begin{table}[h]	 
	\centering
	\caption{Rank of the predicted entire multi-label matrix of each dataset.}\label{tab5} 
	
	\begin{tabular}{@{}cccccc@{}} 
		\toprule
		dataset   & Corel5k     & Espgame   & IAPRTC12   & Pascal07   & Mirflickr   \\ \midrule	
		Rank      & 260         & 268       & 291        & 20         & 38           \\ \bottomrule
	\end{tabular}
\end{table}

\begin{figure}[h] 
	\centering
	\includegraphics[width=0.48 \textwidth, height = 38.1mm]{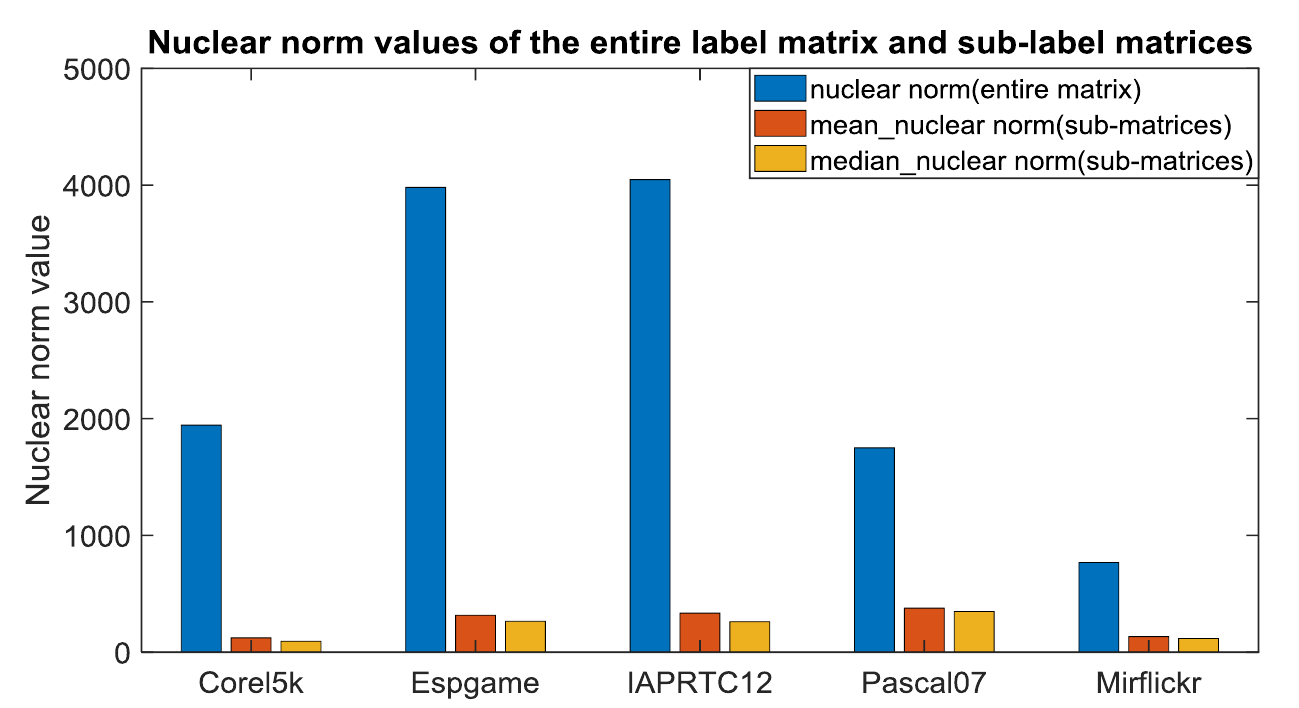} 
	\caption{The nuclear norm of the predicted entire label matrix, the mean value and the median value of the nuclear norm of the predicted sub-label matrices.}\label{Fig3}
	
\end{figure}

\subsection{Ablation Study} \label{V-D}
In this subsection, we study NAIM$^3$L-I (with only the loss function $\mathcal{L}$) 
\begin{equation} \nonumber
\begin{aligned}
{\text {NAIM$^3$L-I :  }} \min _{\mathbf{w}^{(i)}} \frac{1}{2} \sum_{i=1}^{V}\left\|\mathbf{P}^{(i)} \odot\left(\mathbf{X}^{(i)} \mathbf{W}^{(i)}-\mathbf{Y}^{(i)}\right)\right\|_{F}^{2},
\end{aligned}
\end{equation}
and NAIM$^3$L-II (with $\mathcal{L}$  and the first low-rank term in the regularizer $\mathcal{R}$ )
\begin{equation} \nonumber
\begin{aligned}
{\text {NAIM$^3$L-II :  }}  \min _{\mathbf{W}^{(i)}} \frac{1}{2} \sum_{i=1}^{V}\left\|\mathbf{P}^{(i)} \odot\left(\mathbf{X}^{(i)} \mathbf{W}^{(i)}-\mathbf{Y}^{(i)}\right)\right\|_{F}^{2} \\
 +\lambda\left(\sum_{k=1}^{c}\left\|[\mathbf{X}_{k}^{(1)} \mathbf{W}^{(1)} ; \mathbf{X}_{k}^{(2)} \mathbf{W}^{(2)} ; \cdots ; \mathbf{X}_{k}^{(V)} \mathbf{W}^{(V)}]\right\|_{*}\right)
\end{aligned}
\end{equation} 
to validate the effectiveness of the proposed regularizer $\mathcal{R}$, especially the significance of the high-rank term. From Table \ref{tab6}, we can see that NAIM$^3$L-I performs the worst on five datasets while NAIM$^3$L has the best performance. Compared with NAIM$^3$L-I, the performance of NAIM$^3$L-II improves very little, but after adding the high-rank term, all the four metrics raise considerably. This demonstrates that the local and the global structures, especially the latter, are beneficial to characterize the relationships among multiple labels and the presented regularizer  is effective in learning these relations. Additionally, an interesting result is found when comparing Tables \ref{tab2} and \ref{tab6}, that is, NAIM$^3$L-I has better performance than most of other compared methods in Table \ref{tab2}. The reasons may be that other methods make improper assumptions which violate reality and the indicator matrix $\mathbf {P}$ introduced by us can alleviate both the negative effects on missing labels and incomplete views. 

\begin{table}[h]
\centering
\caption{Results of The Variants of NAIM$^3$L with the Ratio of Incomplete Multi-view  $\alpha = 50$\% and the Ratio of Missing Multi-label $\beta = 50$\%. Values in Parentheses Represent the Standard Deviation, and all the Values are Displayed as Percentages. All the Three Methods are Implemented under the Non-aligned Views Setting.}
\label{tab6}
\begin{tabular}{ccccc}
\toprule
datasets                   & metrics  & NAIM$^3$L-I    & NAIM$^3$L-II   & NAIM$^3$L      \\ \midrule
\multirow{4}{*}{Corel5k}   & 1-HL(\%) & \textbf{98.70}(0.00) & \textbf{98.70}(0.00) & \textbf{98.70}(0.01) \\
                           & 1-RL(\%) & 82.73(0.20) & 83.54(0.21) & \textbf{87.84}(0.21) \\
                           & AP(\%)   & 30.20(0.40) & 30.47(0.36) & \textbf{30.88}(0.35) \\
                           & AUC(\%)  & 82.99(0.20) & 83.80(0.21) & \textbf{88.13}(0.20) \\ \midrule
\multirow{4}{*}{Pascal07}  & 1-HL(\%) & 92.83(0.00) & 92.83(0.00) & \textbf{92.84}(0.05) \\
                           & 1-RL(\%) & 77.29(0.18) & 77.35(0.17) & \textbf{78.30}(0.12) \\
                           & AP(\%)   & 48.64(0.35) & 48.66(0.35) & \textbf{48.78}(0.32) \\
                           & AUC(\%)  & 79.99(0.17) & 80.55(0.17) & \textbf{81.09}(0.12) \\ \midrule
\multirow{4}{*}{ESPGame}   & 1-HL(\%) & \textbf{98.26}(0.00) & \textbf{98.26}(0.00) & \textbf{98.26}(0.01) \\
                           & 1-RL(\%) & 79.63(0.20) & 79.80(0.11) & \textbf{81.81}(0.16) \\
                           & AP(\%)   & 24.28(0.20) & 24.34(0.16) & \textbf{24.57}(0.17) \\
                           & AUC(\%)  & 80.04(0.20) & 80.24(0.13) & \textbf{82.36}(0.16) \\ \midrule
\multirow{4}{*}{IAPRTC12}  & 1-HL(\%) & \textbf{98.05}(0.00) & \textbf{98.05}(0.00) & \textbf{98.05}(0.01) \\
                           & 1-RL(\%) & 82.52(0.00) & 82.70(0.00) & \textbf{84.78}(0.11) \\
                           & AP(\%)   & 25.71(0.10) & 25.76(0.10) & \textbf{26.10}(0.13) \\
                           & AUC(\%)  & 82.56(0.10) & 82.76(0.10) & \textbf{84.96}(0.11) \\ \midrule
\multirow{4}{*}{Mirflickr} & 1-HL(\%) & \textbf{88.15}(0.00) & \textbf{88.15}(0.00) & \textbf{88.15}(0.07) \\
                           & 1-RL(\%) & 84.05(0.00) & 84.10(0.00) & \textbf{84.40}(0.09) \\
                           & AP(\%)   & 54.95(0.20) & 54.98(0.16) & \textbf{55.08}(0.18) \\
                           & AUC(\%)  & 83.33(0.00) & 83.39(0.00) & \textbf{83.71}(0.06) \\ \bottomrule
\end{tabular}
\end{table}

\subsection{Statistical Analysis} \label{V-E}
To further analyze the effectiveness of our proposed method, we conduct the significance test on the results reported in Tables 2 and 6. Specifically, for Table 2, we employ the Nemenyi test \cite {nemenyi1963,demvsar2006,zhang2019} on the 20 results (4 metrics on 5 datasets) of six methods and set the significance level $\alpha_{l}$ at 0.05. Then the critical value $q_{\alpha} = 2.850$ and the critical distance $CD = q_{\alpha} \sqrt{k(k+1) / N} (k = 6, N = 20)$. Similarly, for Table 6, $ q_{\alpha} = 2.344$ and $ CD = q_{\alpha} \sqrt{k(k+1) / N} (k = 3, N = 20)$. Results of the Nemenyi test are shown in  Fig. \ref{Fig4a} and \ref{Fig4b}, respectively. As we can see from Fig. \ref{Fig4a}, NAIM$^3$L performs better than the other methods except for the iMVWL method at the significance level 0.05. Besides, from Fig. \ref{Fig4b}, we can conclude that NAIM$^3$L performs better than other two methods at the significance level 0.05, which indicates the rationality and  importance of our regularization term, especially the high-rank term.

\begin{figure}[h]
	\centering
	\subfigure[Nemenyi test of Table 2]{
		\includegraphics[width=0.231\textwidth, height = 25mm]{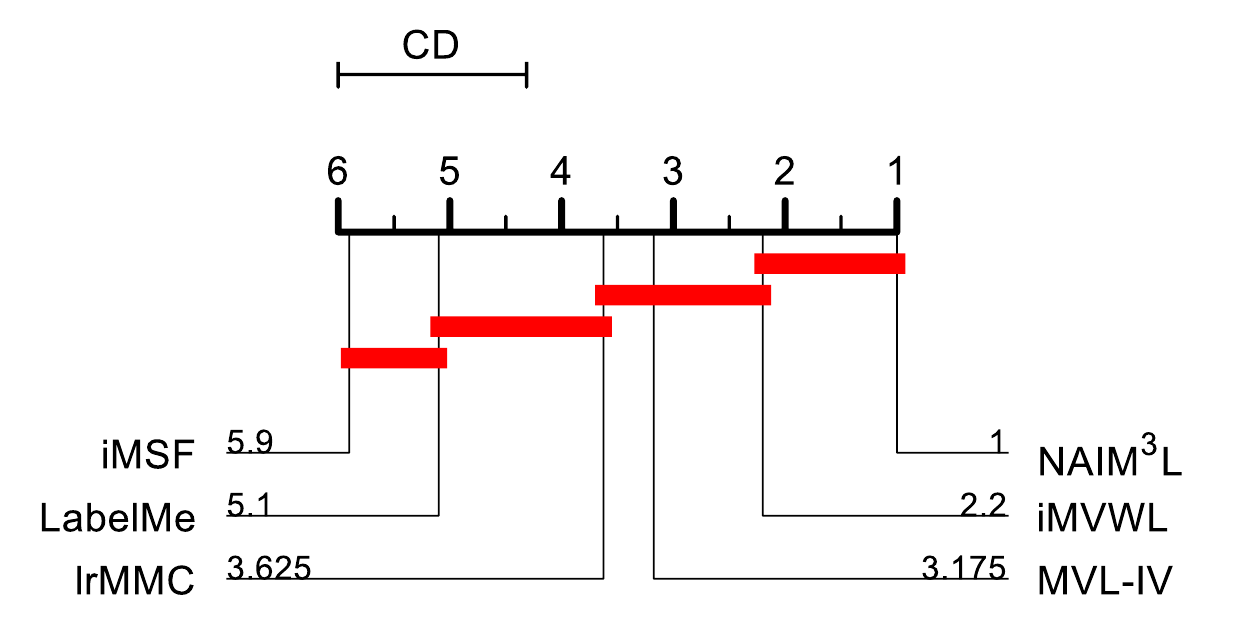}
		\label{Fig4a}
	}
	\subfigure[Nemenyi test of Table 6]{
		\includegraphics[width=0.231\textwidth, height = 25mm]{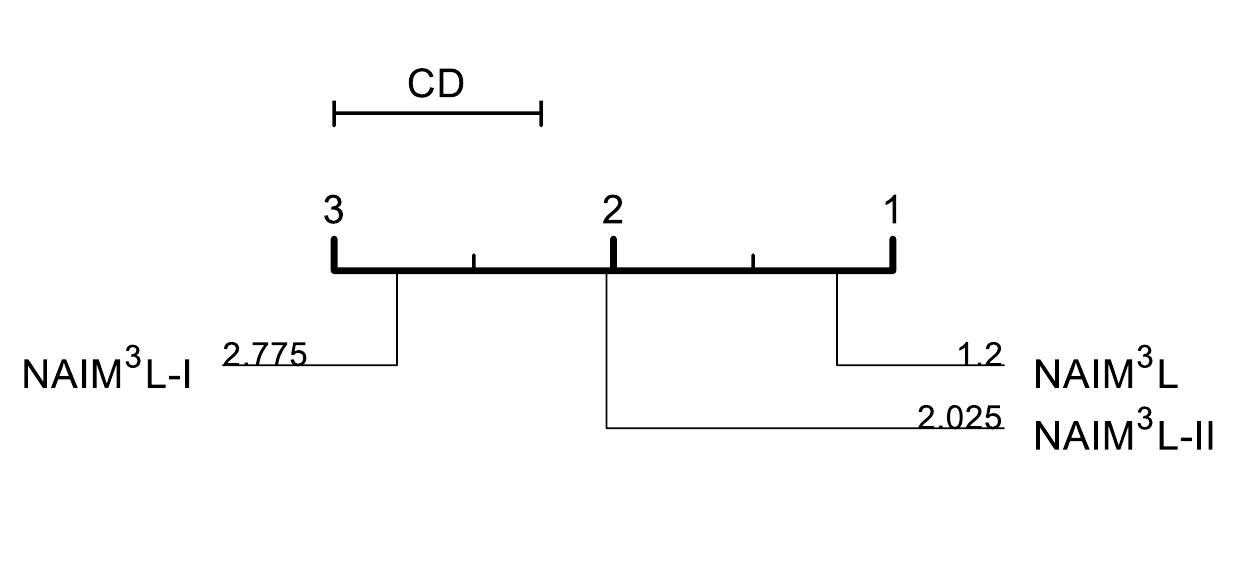}
		\label{Fig4b}
	}
	
	\caption{Statistical analysis of NAIM$^3$L by the Nemenyi test, the significance level is 0.05 and methods not connected are significantly different.}
\end{figure}

\subsection{Different Ratios of Incomplete Multi-view and Missing Multi-label Study}
In this subsection, we conduct experiments on Core15k and Pascal07 to evaluate NAIM$^3$L under different ratios of incomplete multiple views and missing multiple labels. Specifically, to evaluate the impact of different ratios of incomplete multiple views, we fix the ratio of missing multiple labels $\beta = 50$\%, and the ratio of incomplete multiple views $\alpha$ varies within the set of $\{10\%, 30\%, 50\%, 70\%, 90\%\}$. Similarly, to evaluate the impact of different ratios of missing multiple labels, we fix $\alpha=50$\%, and $\beta$ changes within the set of $\{10\%, 30\%, 50\%, 70\%, 90\%\}$. Results of all four metrics on Core15k and Pascal07 are reported in Fig. \ref{Fig5a}-\ref{Fig5b} and \ref{Fig5c}-\ref{Fig5d}, respectively. The ‘Full’ means that the multiple views and multiple labels in the training set are complete. Note that, when the ratio of missing multiple labels $\beta = 50$\% and the ratio of incomplete multiple views $\alpha=90$\%, the multi-view learning principle that each sample in the training set must appear at least one view is violated. Thus, experiment of such incomplete ratio is omitted. We can see from Fig. \ref{Fig5a} and \ref{Fig5c} that only when the ratio of missing multiple labels $\alpha=90$\%, the performance of all four metrics degenerates considerably except for the HL metric. From these observations, we can draw two conclusions, one is that our NAIM$^3$L is relatively robust, the other is that the HL may not be an appropriate metric for multi-label learning. All in all, we can conclude that the lower the incomplete (missing) ratio, the better the performance, which is reasonable and intuitively consistent.

\begin{figure}[h]
\centering
\subfigure[Corel5k]{
\includegraphics[width=0.227\textwidth]{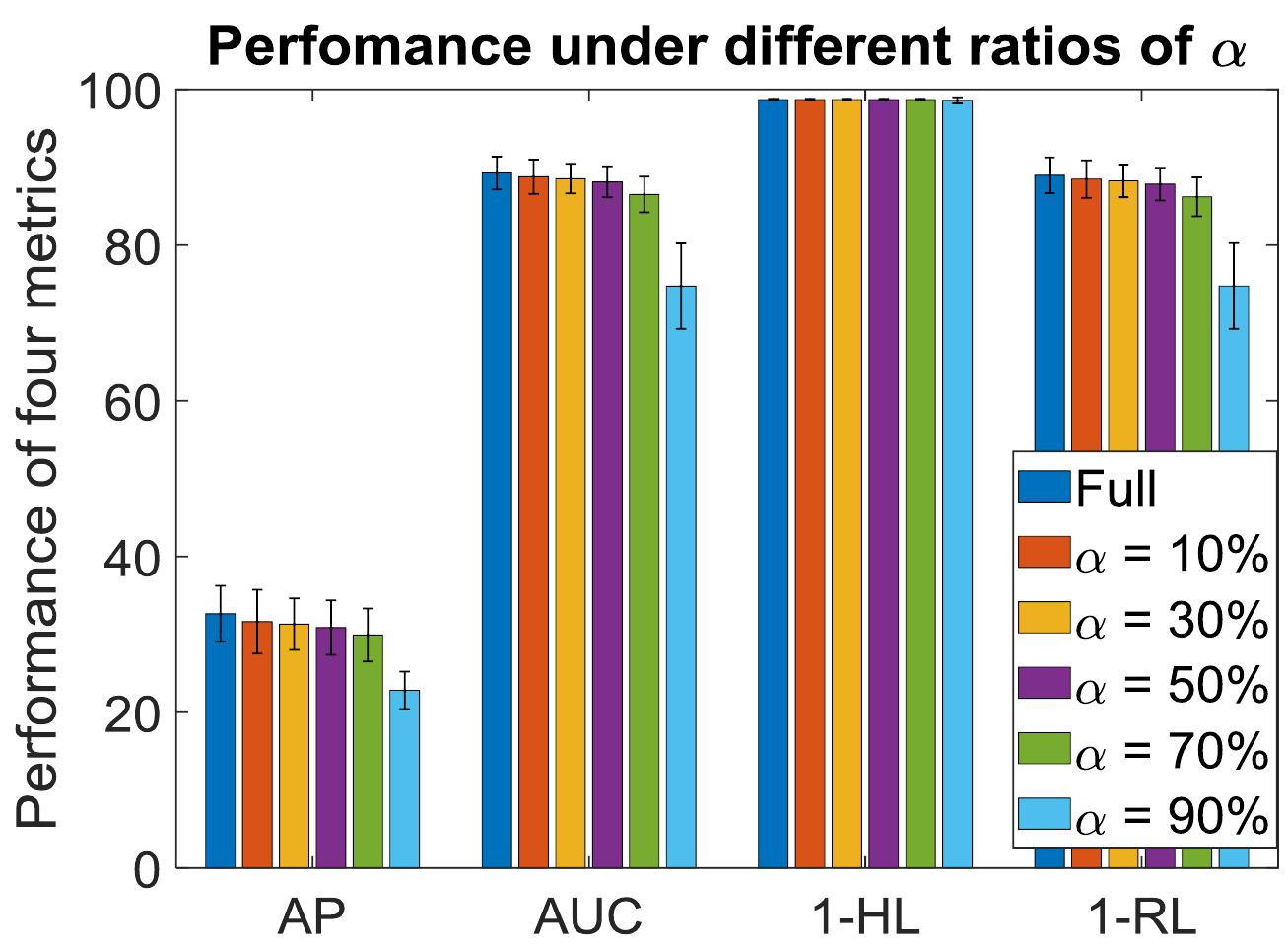}
\label{Fig5a}
}
\subfigure[Corel5k]{
\includegraphics[width=0.227\textwidth]{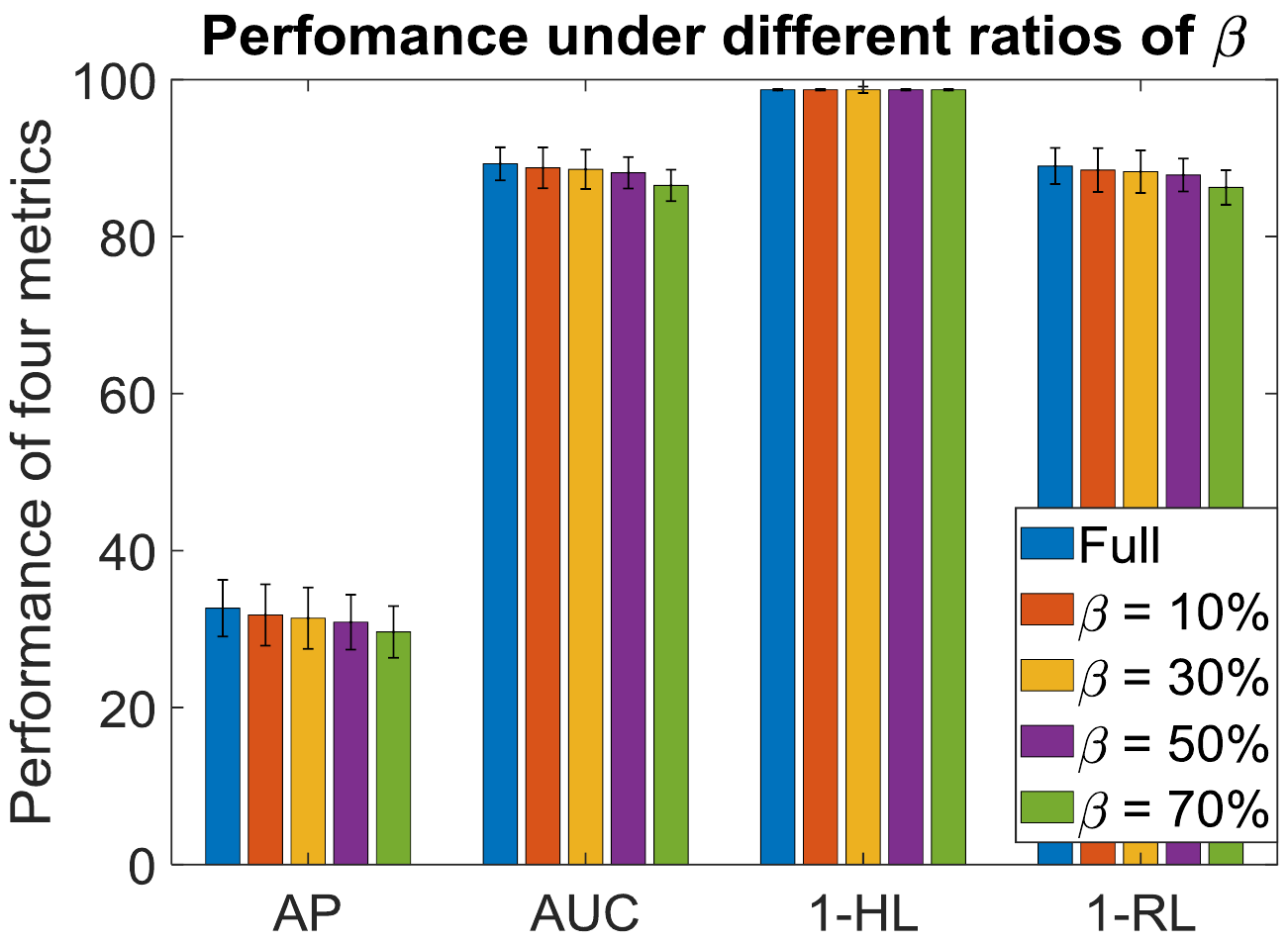}
\label{Fig5b}
}

\subfigure[Pascal07]{
\includegraphics[width=0.227\textwidth]{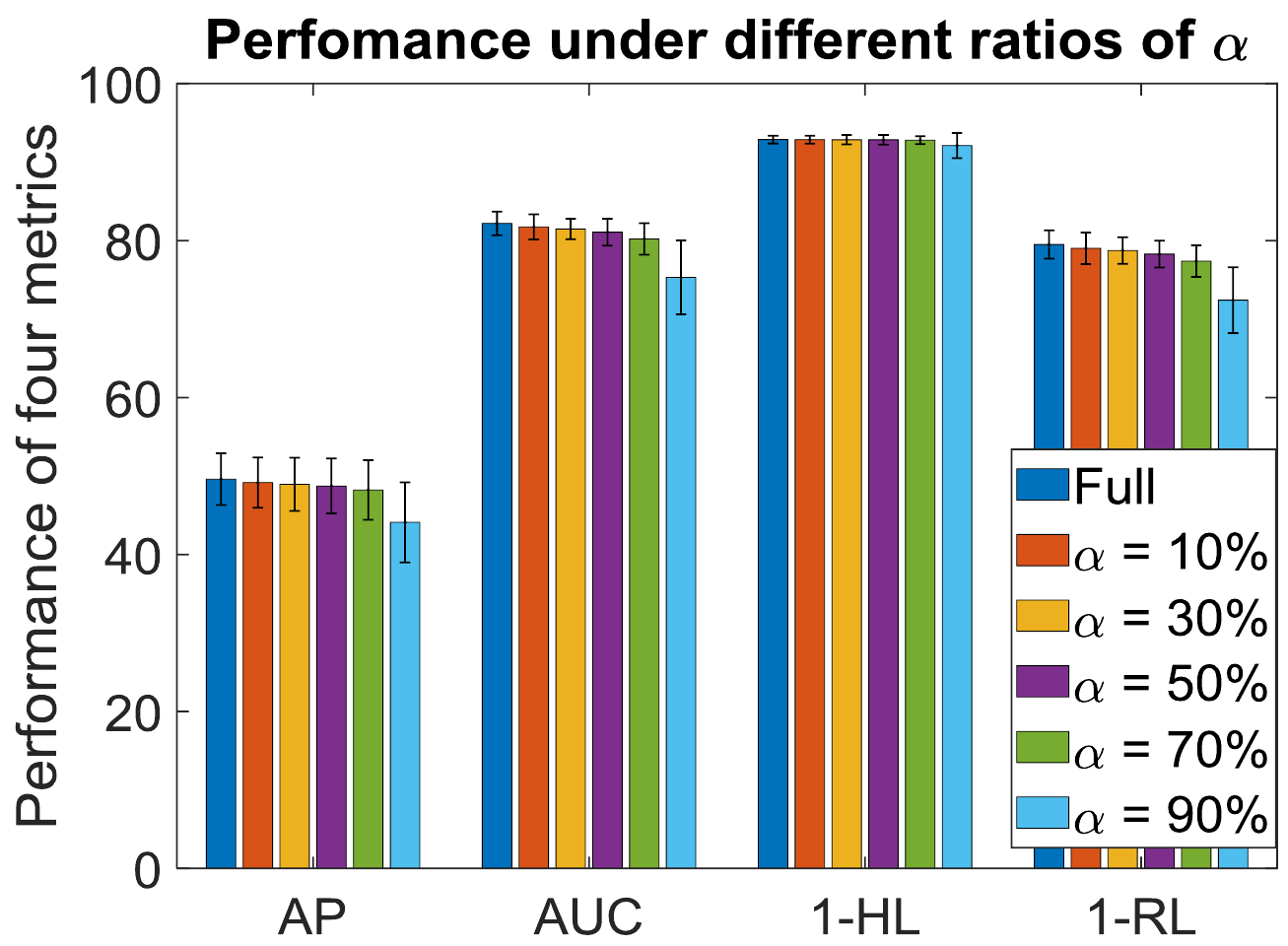}
\label{Fig5c}
}
\subfigure[Pascal07]{
\includegraphics[width=0.227\textwidth]{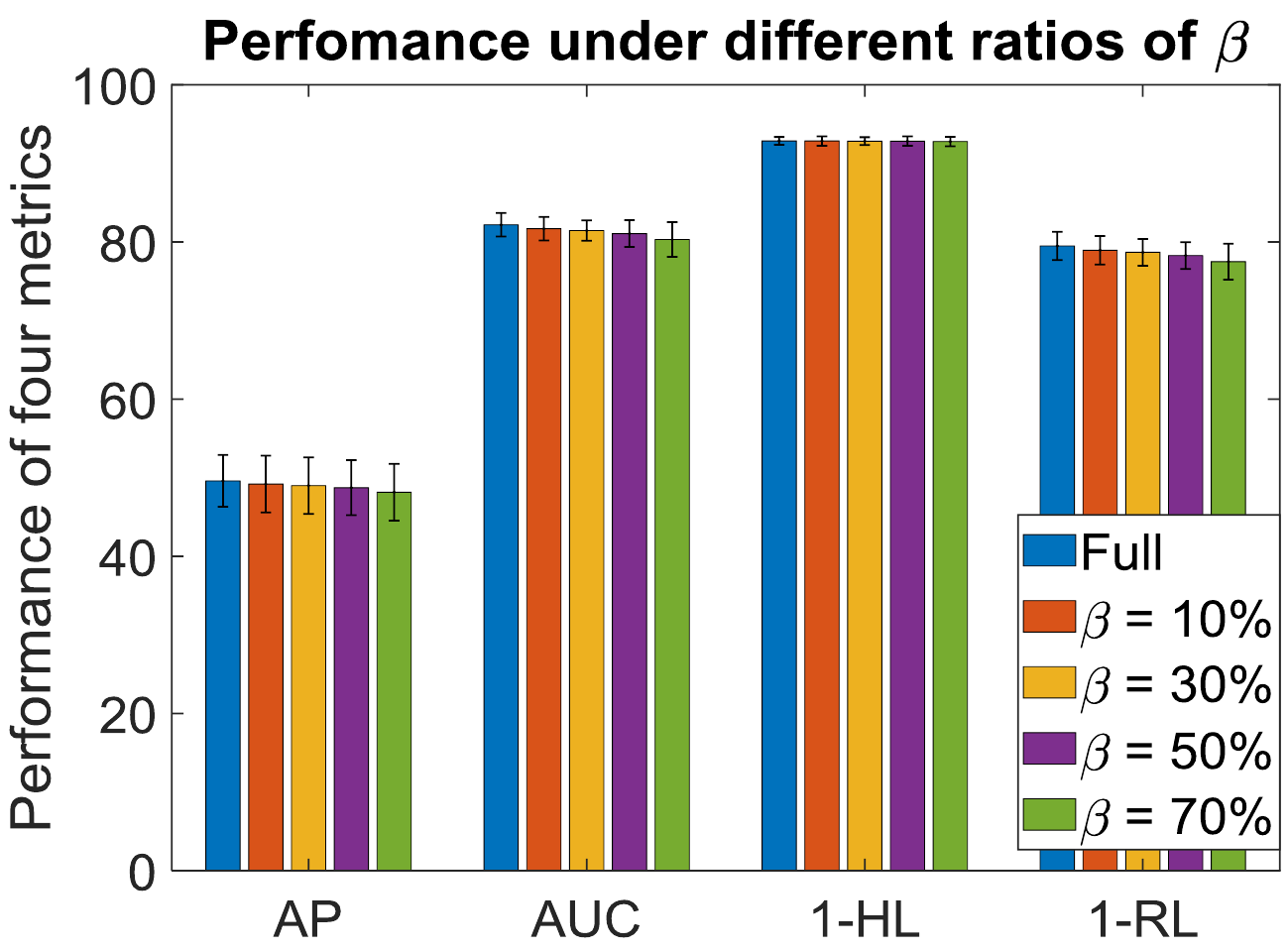}
\label{Fig5d}
}
\caption{The performance of four metrics on the Corel5k dataset and the Pascal07 dataset under different ratios of incomplete multiple views ($\alpha$) and missing multiple labels ($\beta$). The ‘Full’ means that the multiple views and multiple labels in the training set are complete. The average value and standard deviation are shown in each sub-figures.}
\end{figure}

\subsection{Hyper-parameter Study}
There is only one hyper-parameter in NAIM$^3$L, thus it is easy to choose the relatively optimal hyper-parameter. The grid search technique is adopted to choose the optimal hyper-parameter within the set of $\{10^{-3}, 10^{-2}, 0.1, 1, 10, 100\}$. For the sake of clarity, we scale them by $log10$ when showing in figures. Besides, we narrow the scope of Y-axis to make the trends of the performance look clearer. From Fig. \ref{Fig6a}-\ref{Fig6d}, we can see that NAIM$^3$L achieves relatively good performance when $\lambda$ in [0.1, 1], and the values of the four metrics vary slightly, which indicates that our method is NOT sensitive to the hyper-parameter $\lambda$. 

\begin{figure}[H]
	\centering
	\subfigure[]{
		\includegraphics[width=0.227\textwidth]{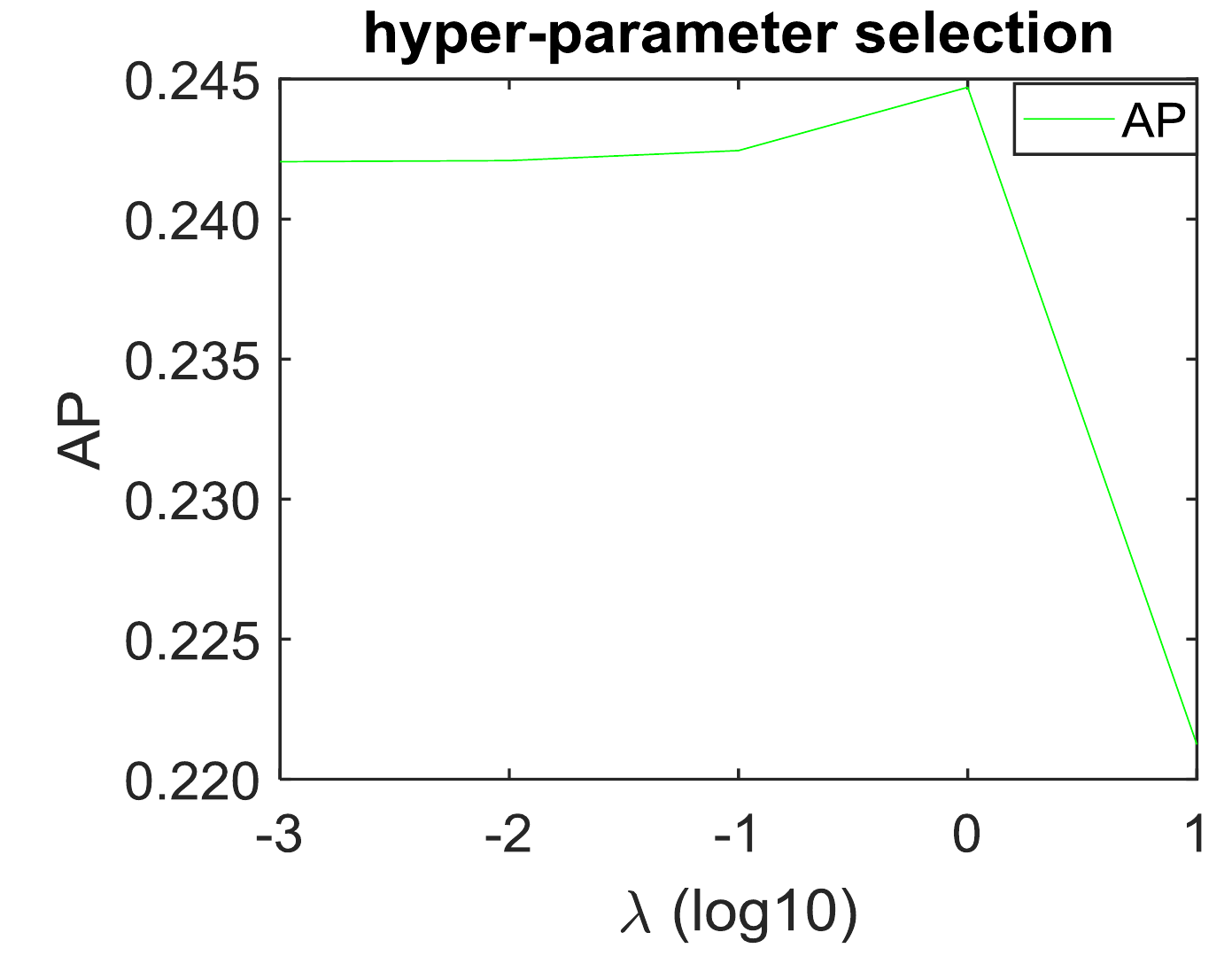}
		\label{Fig6a}
	}
	\subfigure[]{
		\includegraphics[width=0.227\textwidth]{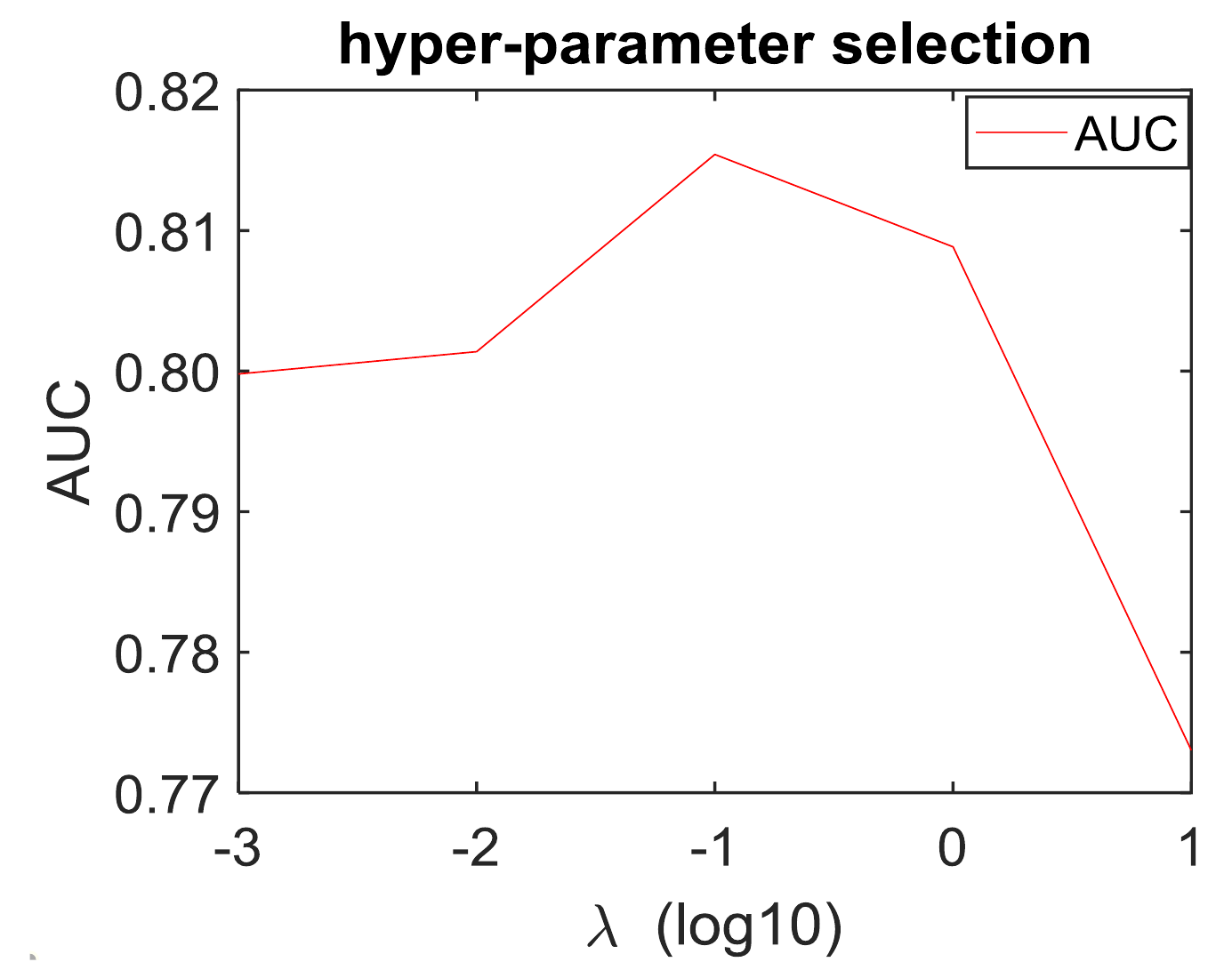}
		\label{Fig6b}
	}	
	\subfigure[]{
		\includegraphics[width=0.227\textwidth]{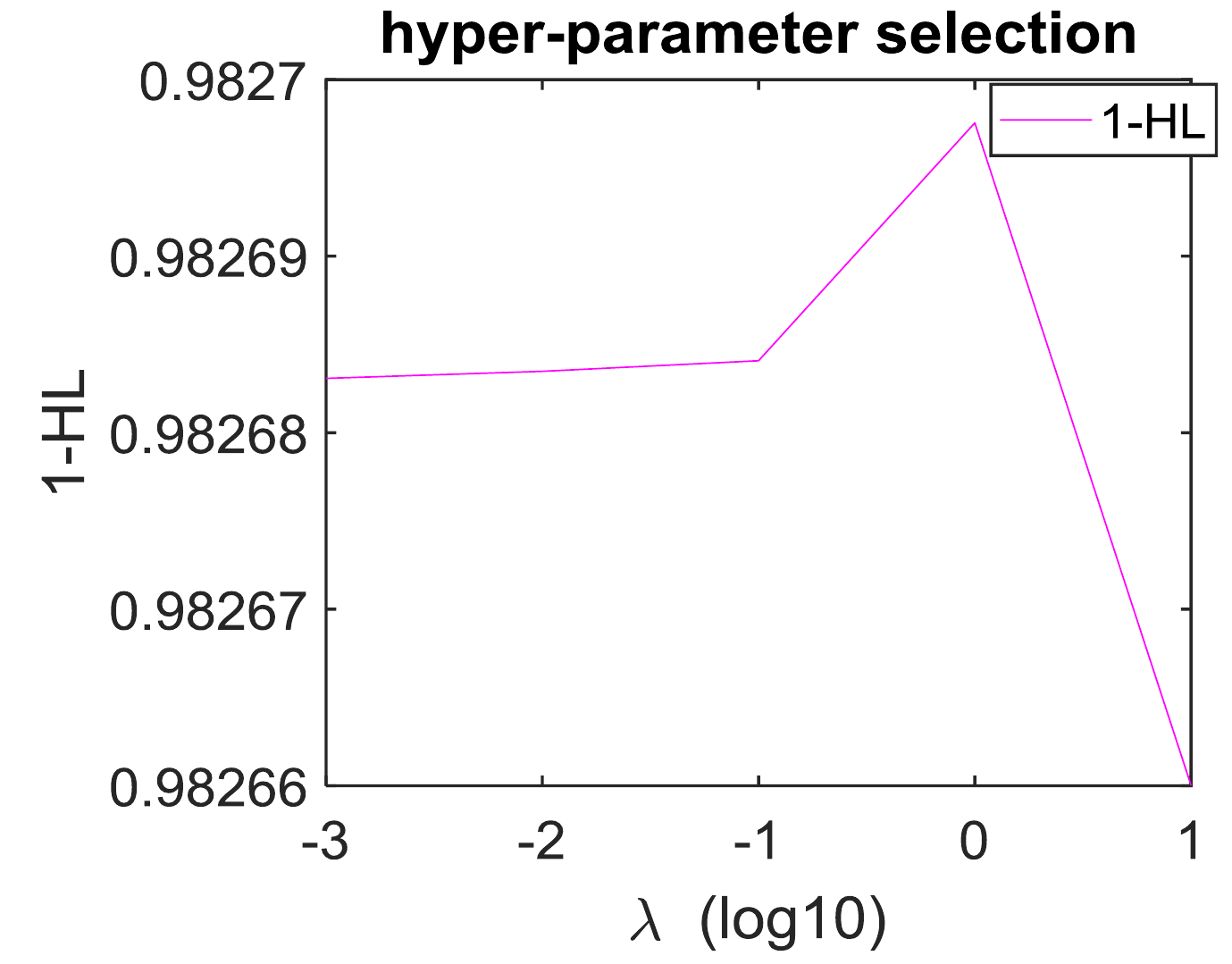}
		\label{Fig6c}
	}
	\subfigure[]{
		\includegraphics[width=0.227\textwidth]{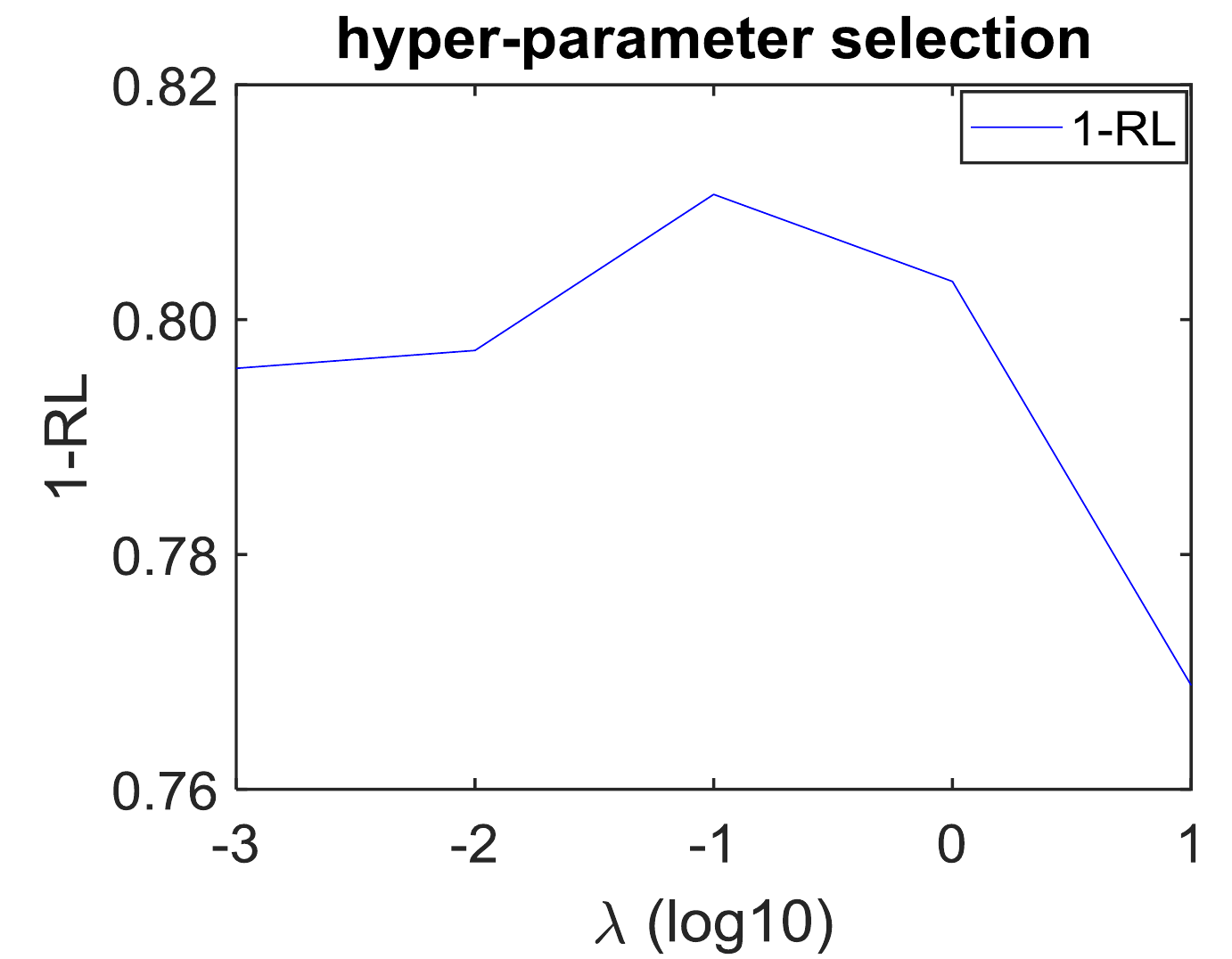}
		\label{Fig6d}
	}
	\caption{Results of the hyper-parameter selection on four metrics with different $\lambda$. The values of $\lambda$ are scaled by $log10$ for clarity.}
\end{figure}

\subsection{Impacts Study of $\mu$} 
In this subsection, we conduct experiments to validate that the parameter $\mu$ introduced in the ADMM algorithm does not change the performance of our model. From Fig. \ref{Fig7a} we find that all the four metrics remain the same when $\mu$ changes, and from  Fig. \ref{Fig7b} we can see that $\mu$ only impacts the convergence rate. The larger the $\mu$ is, the slower the objective function converges. Thus, $\mu$ is a parameter that does not need to be adjusted. Except for the experiments in this subsection, $\mu$ is fixed at 5 throughout all other experiments in this article.

\begin{figure}[h]
\centering
\subfigure[]{
\includegraphics[width=0.227\textwidth]{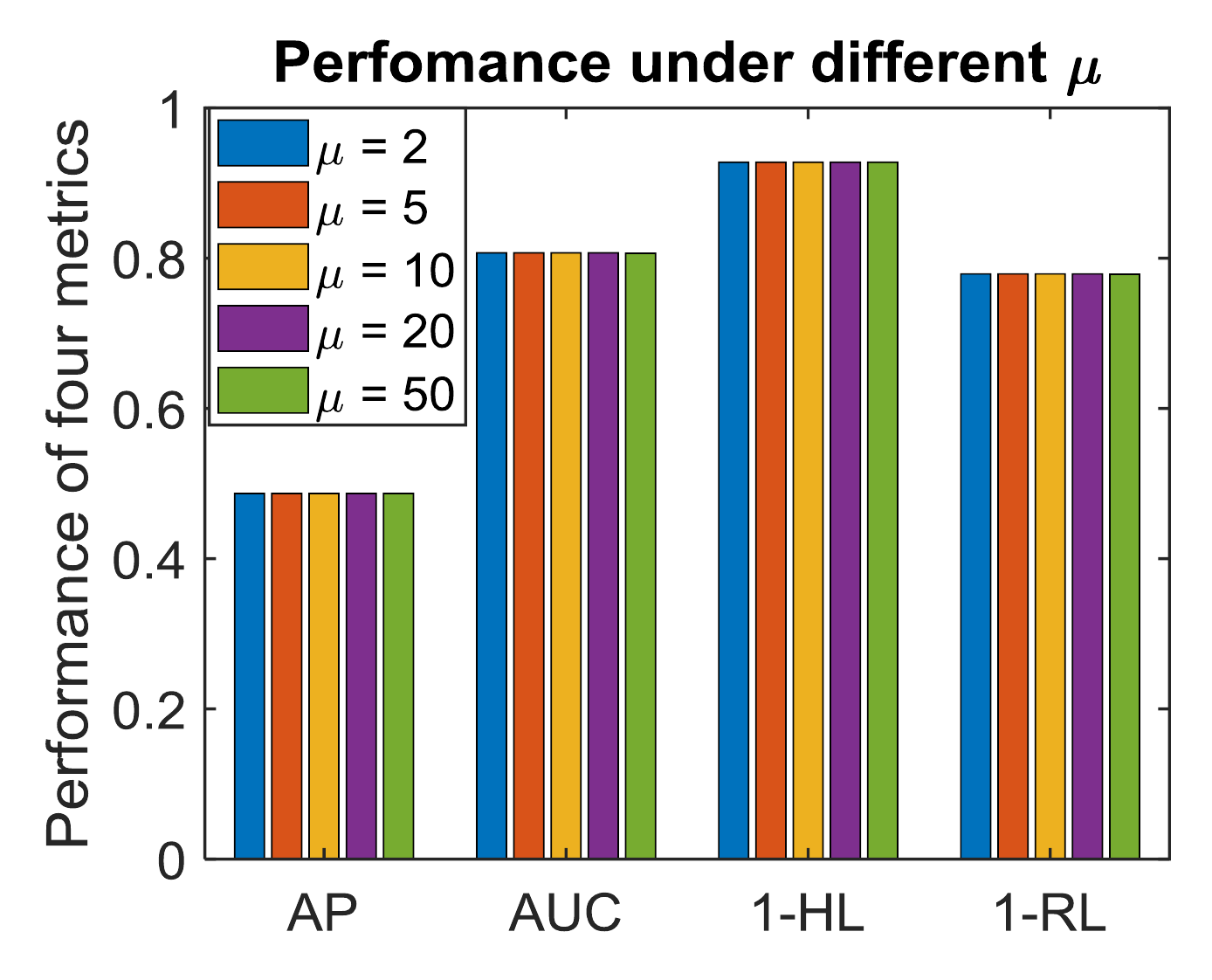}
\label{Fig7a}
}
\subfigure[]{
\includegraphics[width=0.227\textwidth]{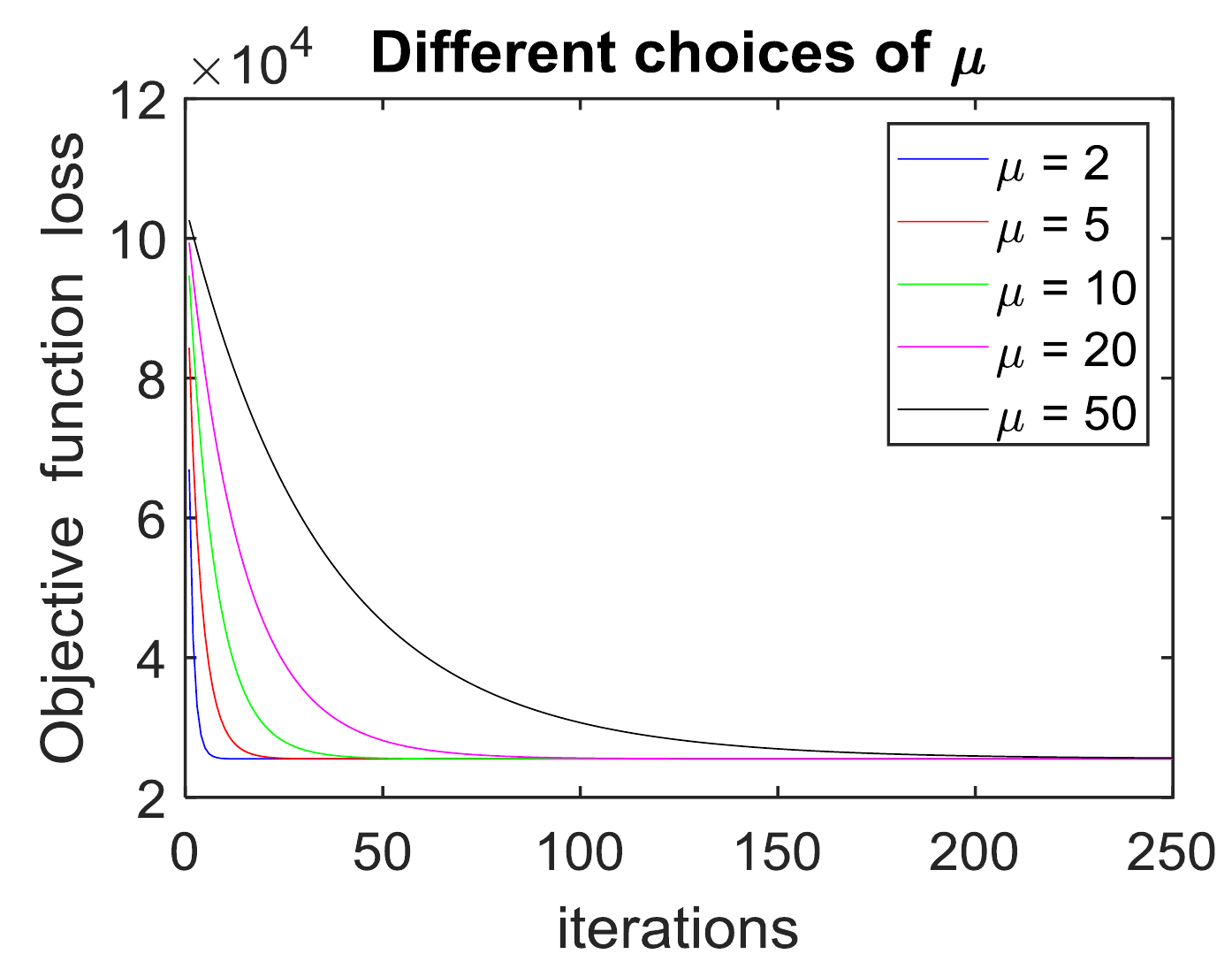}
\label{Fig7b}
}
\caption{The performance of all four metrics and the convergence rates under different $\mu$.}
\end{figure}

\subsection{Convergence Study}
In this subsection, we show the convergence curves of the objective function \eqref{Eq5} and the surrogate objective function \eqref{Eq11} to validate the convergence analysis given in subsection \ref{IV-C}. In Fig. \ref{Fig8a} and \ref{Fig8b}, we can see that both of the objective functions converge after about 20 iterations and the differences between them are small, which means the surrogate function is an appropriate alternate.

\begin{figure}[h]
\centering
\subfigure[]{
\includegraphics[width=0.227\textwidth]{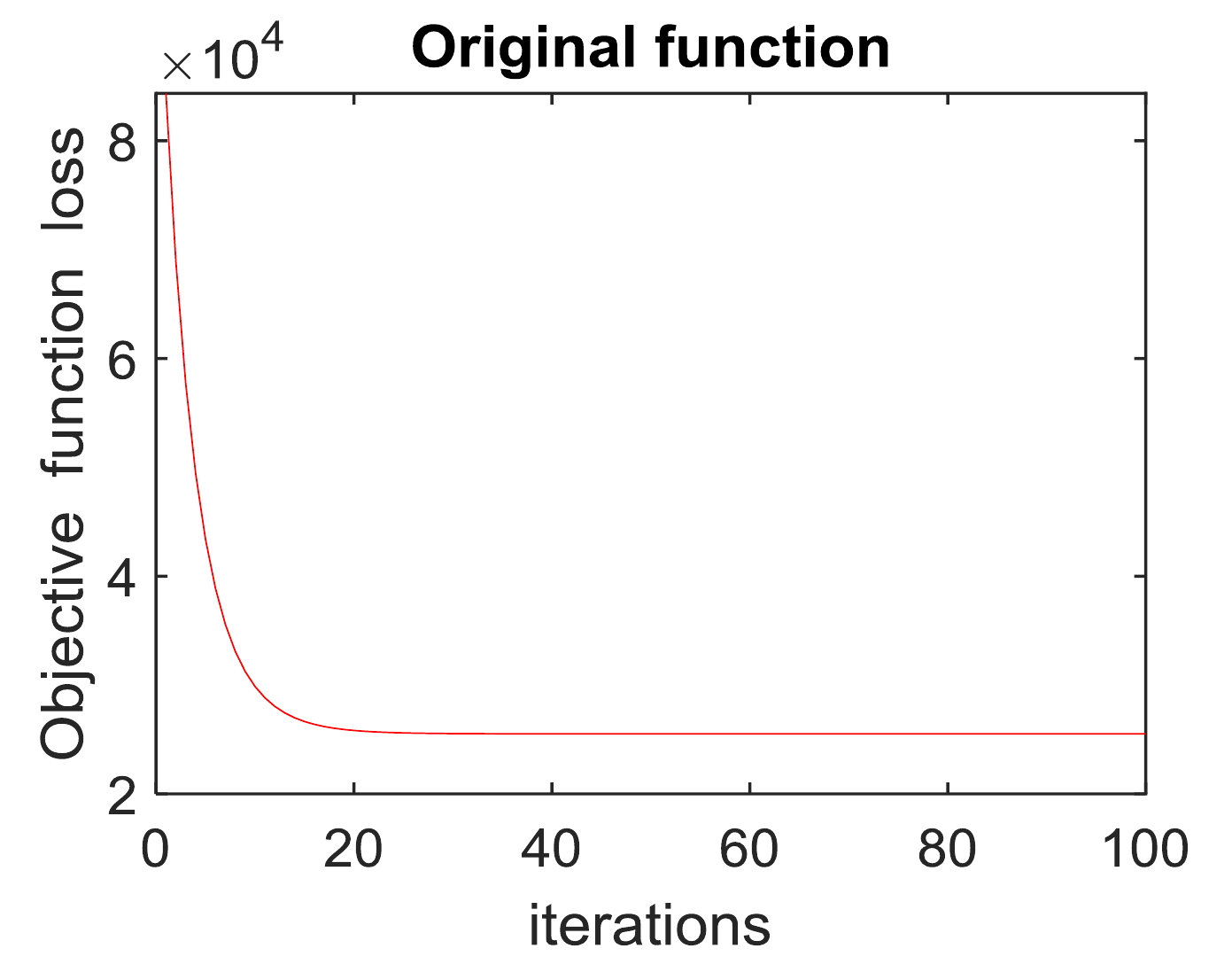}
\label{Fig8a}
}
\subfigure[]{
\includegraphics[width=0.227\textwidth]{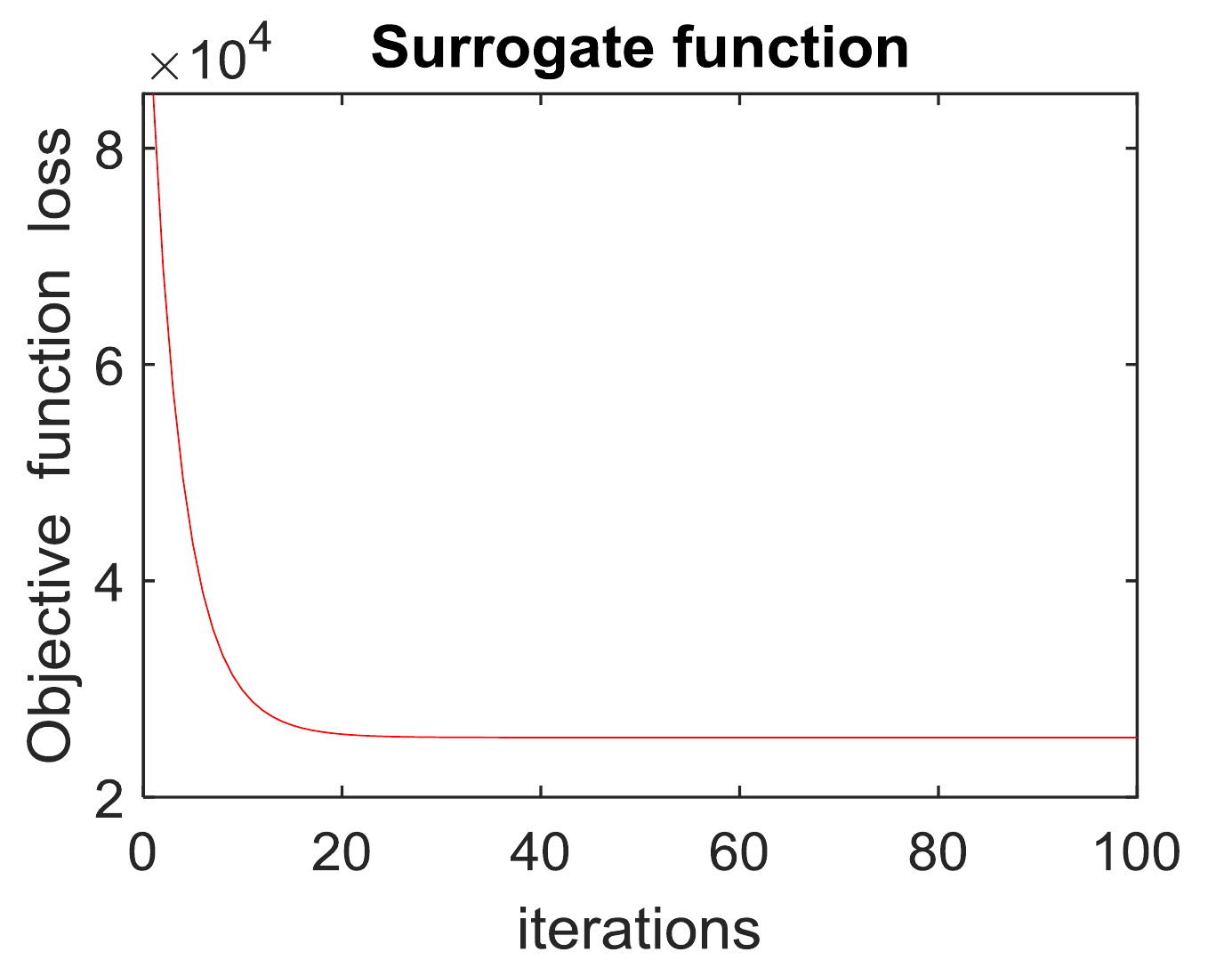}
\label{Fig8b}
}
\caption{Convergence curves of the original objective function and the surrogate objective function.}
\end{figure}

\subsection{Efficiency Study of Algorithm 2} \label{5.9}

In this subsection, we conduct several experiments for computing the sub-gradient of the trace norm with matrices of different sizes to validate the high efficiency of Algorithm 2. Random matrices are used and the matrix size varies from $10000 \times 100$ to $100000 \times 400$. Algorithm 2 is compared with three widely used SVD methods, that is, the full SVD, the economical SVD, and the truncated SVD. The truncated SVD is implemented with parameter r (rank of the matrix), also known as reduced SVD in this case. All the experiments are repeated ten times and the average results of the run time (in seconds) are reported in Table \ref{tab7}. The SVD (full) fails when the matrix size increase to $40000 \times 100 $ because of that the computer is out of memory (16G RAM). From Table \ref{tab7}, we can find that Algorithm 2 is of orders of magnitude faster than the SVD( full). Besides, both the run time of SVD (econ) and SVD (trunc) is about twice of the Algorithm 2. All of these illustrate the efficiency of the Algorithm 2.

\begin{table*}[htbp]
\centering
\caption{The Average Run Time of Different Methods for Computing the sub-Gradient of the Trace Norm (in Seconds). '-' Denotes that the Computer is out of Memory.}
\label{tab7}
\begin{tabular}{@{}cccccccccc@{}}
\toprule
matrix   size & 10000$\times$100 & 10000$\times$200 & 20000$\times$100 & 20000$\times$200 & 40000$\times$100 & 40000$\times$200 & 40000$\times$400 & 100000$\times$200 & 100000$\times$400 \\ \midrule
SVD (full)     & 20.0031         & 1.3839          & 80.2122         & 4.9386          & -               & -               & -               & -               & -               \\
SVD (econ)     & 0.0571          & 0.1377          & 0.1459          & 0.2264          & 0.1696          & 0.4400          & 1.3460           & 1.3813          & 3.9416          \\
SVD (trunc)   & 0.0475          & 0.1471          & 0.0955          & 0.2133          & 0.1838          & 0.4791          & 1.4396          & 1.4526          & 4.2048          \\
Algorithm 2   & \textbf{0.0334} & \textbf{0.0731} & \textbf{0.0590} & \textbf{0.1603} & \textbf{0.0864} & \textbf{0.2109} & \textbf{0.6736} & \textbf{0.5019} & \textbf{1.6042} \\
\bottomrule  
\end{tabular}
\end{table*}

\section{Conclusion}  \label{secVI}
In this paper, we propose a concise yet effective model called NAIM$^3$L to simultaneously tackle the missing labels, incomplete views and non-aligned views challenges with only one hyper-parameter in the objective function. An efficient ADMM algorithm with linear computational complexity regarding the number of samples is derived. Besides, NAIM$^3$L outperforms state-of-the-arts on five real data sets even without view-alignment. Note that, this framework can be directly non-linearized to its kernerlized version and can cooperate with deep neural network as our future research.

\ifCLASSOPTIONcompsoc
  \section*{Acknowledgments}
\else
  \section*{Acknowledgment}
\fi

This work is supported by the Key Program of NSFC under Grant No. 61732006 and NSFC under Grant No. 61672281. The authors would like to thank Dr. Huan Li and Zhenghao Tan for their generous help and beneficial discussions.

\ifCLASSOPTIONcaptionsoff
  \newpage
\fi



\bibliographystyle{IEEEtran}
\bibliography{IEEEabrv,PAMI_ref}
\begin{IEEEbiography}[{\includegraphics[width=1in,height=1.25in,clip,keepaspectratio]{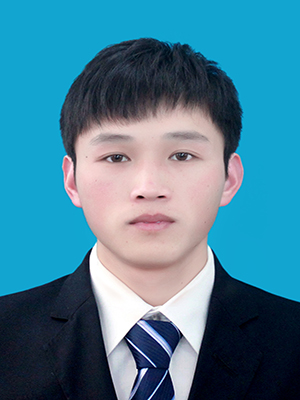}}]{Xiang Li}
received his B.S. degree in mathematics from College of Science, Nanjing University of Aeronautics and Astronautics in 2012. And he completed his M.S. degree in applied mathematics from College of Science, Nanjing University of Aeronautics and Astronautics in 2018. 

Now he is a fourth-year Ph.D. candidate in College of Computer Science and Technology, Nanjing University of Aeronautics and Astronautics. His research interests include multi-view learning and multi-label learning.
\end{IEEEbiography}

\begin{IEEEbiography}[{\includegraphics[width=1in,height=1.25in,clip,keepaspectratio]{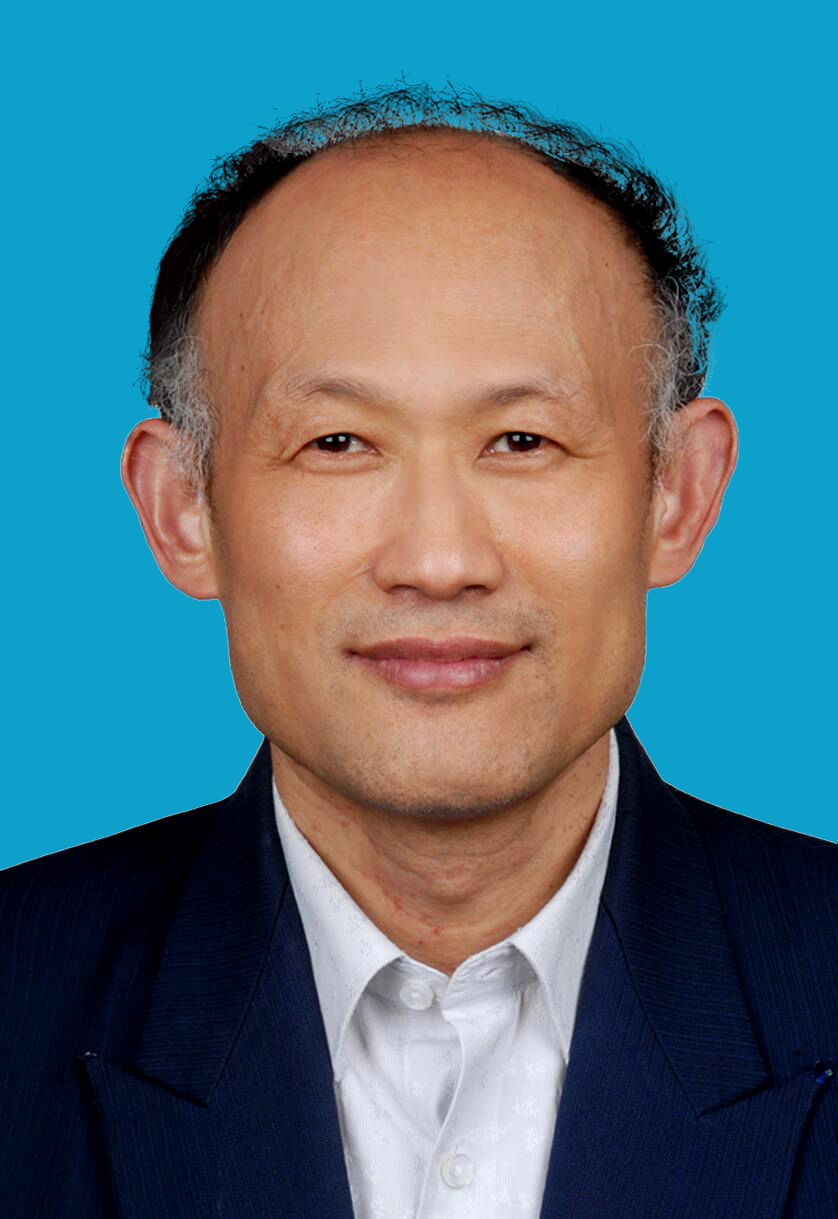}}]{Songcan Chen}
received his B.S. degree in mathematics from Hangzhou University (now merged into Zhejiang University) in 1983. In 1985, he completed his M.S. degree in computer
applications at Shanghai Jiaotong University and then worked at Nanjing University of Aeronautics and Astronautics (NUAA) in January 1986, where he received a Ph.D. degree in communication and information systems in 1997.

Since 1998, as a full-time professor, he has been with the College of Computer Science and Technology at NUAA. His research interests include pattern recognition, machine learning and neural computing. He has published over 100 top-tier journals, such as IEEE Transactions on Pattern Analysis and Machine Intelligence, IEEE Transactions on Knowledge and Data Engineering, IEEE Transactions on Image Processing, IEEE Transactions on Neural Networks and Learning Systems, IEEE Transactions on Fuzzy Systems, IEEE Transactions on Information Forensics \& Security, IEEE Transactions on Systems, Man \& Cybernetics-Part B, IEEE Transactions on Wireless Communication and conference papers, such as International Conference on Machine Learning, IEEE Conference on Computer Vision and Pattern Recognition, International Joint Conference on Artificial Intelligence, AAAI Conference on Artificial Intelligence, IEEE International Conference on Data Mining and so on. He is also an IAPR Fellow.
\end{IEEEbiography}

\end{document}